\PassOptionsToPackage{table,dvipsnames}{xcolor}
\documentclass{muxlab}

\usepackage{amsmath,amsfonts,bm}

\usepackage[dvipsnames]{xcolor}
\usepackage[textsize=scriptsize]{todonotes}

\newif\ifshowcomments
\showcommentstrue

\ifshowcomments
  \newcommand{\alperen}[2][]{%
    \todo[
      color=blue!15,
      linecolor=blue!70!black,
      bordercolor=blue!70!black,
      #1
    ]{\textbf{Alperen:} #2}%
  }

  \newcommand{\alpereninline}[2][]{%
    \todo[
      inline,
      color=blue!10,
      linecolor=blue!70!black,
      bordercolor=blue!70!black,
      #1
    ]{\textbf{Alperen:} #2}%
  }
\else
  \newcommand{\alperen}[2][]{}
  \newcommand{\alpereninline}[2][]{}
\fi

\def\eqref#1{equation~\ref{#1}}

\def\1{\bm{1}}

\DeclareMathAlphabet{\mathsfit}{\encodingdefault}{\sfdefault}{m}{sl}
\SetMathAlphabet{\mathsfit}{bold}{\encodingdefault}{\sfdefault}{bx}{n}

\newcommand{\R}{\mathbb{R}}

\newcommand{\softmax}{\mathrm{softmax}}

\renewcommand*{\backref}[1]{}
\renewcommand*{\backrefalt}[4]{\ifcase#1\relax\or(page #2)\else(pages #2)\fi}

\usepackage{amsmath,amssymb,amsthm}
\usepackage{mathtools}
\usepackage{thmtools,thm-restate}
\newtheorem{theorem}{Theorem}
\newtheorem{proposition}[theorem]{Proposition}
\newtheorem{lemma}[theorem]{Lemma}
\newtheorem{corollary}[theorem]{Corollary}
\theoremstyle{definition}
\newtheorem{definition}[theorem]{Definition}
\theoremstyle{remark}

\AddToHook{cmd/appendix/before}{%
    \crefalias{section}{appendix}%
    \crefalias{subsection}{appendix}
}
\renewcommand{\eqref}[1]{(\ref{#1})}
\usepackage{url}
\usepackage{wrapfig}
\usepackage{array}
\usepackage{comment}
\usepackage{tabularx,makecell}

\definecolor{OursBg}{HTML}{EDE7F6}
\definecolor{LinkBlue}{HTML}{1A0DAB}

\newcommand{\cmark}{\textcolor{green!50!black}{\ding{51}}}
\newcommand{\xmark}{\textcolor{red!70!black}{\ding{55}}}

\newcommand{\ours}{\textsc{MUX}}
\newcommand{\muxop}{\mathsf{mux}}
\newcommand{\onehot}{\mathrm{onehot}}
\newcommand{\Vocab}{\mathcal V}
\newcommand{\Esep}{\mathcal E}
\newcommand{\ind}{\mathbf 1}

\usepackage{colortbl}

\title{\ours{}: Continuous Reasoning via Multiplexed Tokens}

\renewcommand{\authorlist}{%
  {\sffamily\color{muxfg}%
    \textbf{Ayhan Suleymanzade}\textsuperscript{1},\;\;
    \textbf{Halil Alperen Gozeten}\textsuperscript{2},\;\;
    \textbf{Michael Bronstein}\textsuperscript{3,1}%
  }\\[2pt]%
  {\sffamily\color{muxfg}%
    \textbf{\.{I}smail \.{I}lkan Ceylan}\textsuperscript{4,1,3,\dag},\;\;
    \textbf{Jinwoo Kim}\textsuperscript{5,\dag}%
  }%
}

\renewcommand{\affiliationlist}{%
  {\small\textsuperscript{1}AITHYRA\quad \textsuperscript{2}University of Michigan\quad \textsuperscript{3}University of Oxford\quad \textsuperscript{4}TU Wien\quad \textsuperscript{5}KAIST}\\[2pt]%
  {\footnotesize $^\dag$Equal advising}
}

\abstract{Language models solve complex problems by articulating intermediate reasoning steps in natural language. While effective, this process is computationally bottlenecked: each reasoning step conveys only a single subword, and many are spent expressing a thought instead of carrying out computation. We propose \ours{}, a simple method for high-bandwidth and compact reasoning based on distillation of discrete reasoning into continuous multiplexed tokens in a latent space. Here, each latent token is trained to represent a weighted linear superposition (multiplexing) of a span of discrete reasoning subwords, where this superposition is lossless by construction and the span can be fully recovered (demultiplexing). We prove that simple position-dependent weightings, such as suitable geometric decay, support lossless multiplexing, which in turn prevents shortcut behaviors caused by latent collapse. We further show that multiplexed reasoning can perform parallel exploration in problems that require search. Across 32 evaluation settings spanning four language models, \ours{} outperforms strong latent reasoning baselines. Ablation and probing analyses further show that the learned latent tokens encode faithful and interpretable reasoning. Our results suggest that lossless superposition as local learning targets constitutes a sufficient condition for achieving strong and efficient latent continuous reasoning.%
\\[4pt]
\textbf{Code:} \url{https://github.com/MisakiTaro0414/mux}}

\begin{document}

\maketitle

\section{Introduction}

Modern language models are capable of solving complex problems in domains such as mathematics, coding, and commonsense tasks through their \emph{reasoning} mechanism~\citep{hurst2024gpt, team2023gemini, touvron2023llama}. In autoregressive language models, this mechanism typically involves verbalizing intermediate solution steps in natural language before producing the final answer~\citep{nye2021show,wei2022chain, kojima2022large}. However, this mode of operation imposes a strict constraint on the computational bandwidth since each reasoning step transmits only a single subword.  Moreover, many of these steps are redundant~\citep{xia2025tokenskip, li2026making}, since the model mirrors problem-solving patterns learned from human-generated corpora, which are inherently more optimized for communication than computation. These limitations motivate the development of approaches that enable higher-bandwidth and more compact reasoning in language models.

Reasoning in continuous latent spaces has emerged as an alternative paradigm, where a language model sequentially predicts continuous vectors instead of subwords before answering~\citep{hao2025training, xu2025softcot}. These latent reasoning approaches have high bandwidths, since each step can convey multiple subwords simultaneously by encoding them \emph{in superposition}. This notably enables exploring different problem-solving paths in parallel, offering potential improvements in planning and search tasks~\citep{zhu2026reasoning, gozeten2026continuous}. Despite such potential, latent reasoning methods have not yet been widely adopted, in part because they are notoriously hard to learn. One class of methods relies on temporal backpropagation of \emph{trajectory-level losses}~\citep{hao2025training, shen2025codi}, which tend to produce shortcut or uninformative latent tokens~\citep{zhang2025latent, cui2026latent}. Other approaches define \emph{local distillation losses} for each latent token based on discrete reasoning traces~\citep{wei2026simcot, kuzina2026kava}. These methods avoid shortcuts, but at the cost of additional technical complexity such as autoregressive decoders or cache compression, as well as potentially restricting the ability to maintain diverse hypotheses needed for search~\citep{cui2026latent}. Such apparent tradeoff motivates our key question: \emph{What should constitute the supervision target for continuous latent reasoning?}

\begin{figure}[t]
\centering
\includegraphics[width=0.9\textwidth]{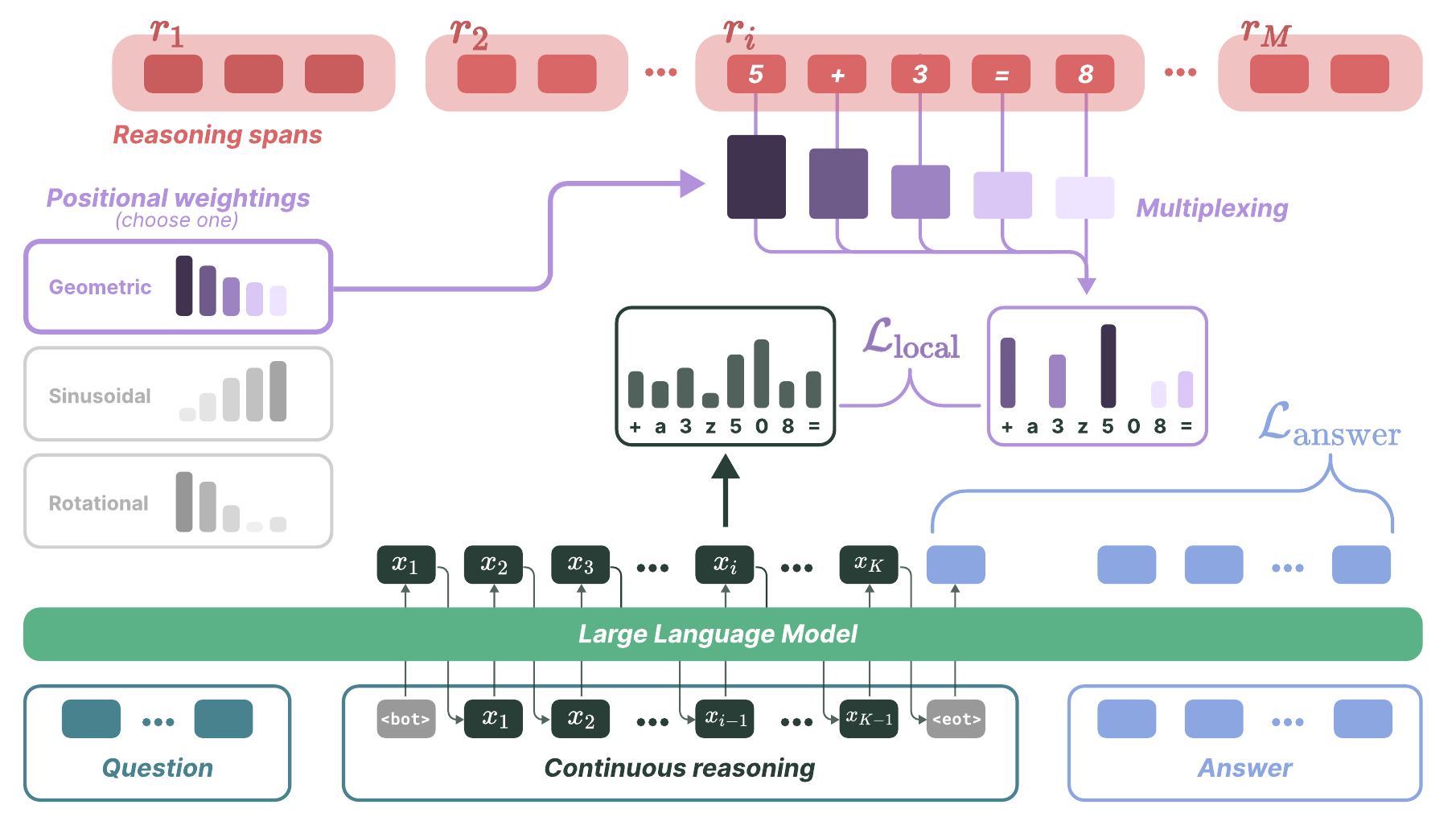}
\caption{\textbf{Overview of \ours.}
Given a question, the language model predicts a sequence of continuous latent reasoning tokens. Each token is linearly projected to the vocabulary space and is trained to represent a local span of discrete reasoning steps.
We construct the learning target from each discrete span as a weighted average of one-hot encodings, and train each latent token with local KL divergence loss.
The answer is trained with standard cross-entropy loss.
}
\label{fig:overview}
\end{figure}

We address this question with \ours{}, a simple and novel training method for high-bandwidth, compact latent reasoning in language models.
Here, we view the local distillation setting, where each latent token is supervised to represent a span of discrete reasoning steps, as learning a fixed-dimensional, continuous signal that encodes a varying-length categorical signal.
This naturally connects to \emph{multiplexing} in communication systems, which allows multiple logical signals to share a common physical medium.
Drawing inspiration from code-division multiplexing~\citep{fan2020signature}, we propose to leverage a multiplexed encoding for variable-length categorical signals based on \emph{linear superpositions} of one-hot encodings.
Our hypothesis is that (i) local distillation from such multiplexed targets induces faithful latent reasoning, provided that these targets are \emph{lossless} encodings of discrete reasoning spans, while (ii) natively supporting joint encoding of multiple possibilities, thanks to their superposed construction.
This is also conceptually simpler than prior methods, as it does not require an auxiliary autoregressive decoder or a compressed cache target.
Our main contributions can be summarized as:
\begin{enumerate}[label=(\arabic*),leftmargin=*]
\item \textbf{Latent reasoning via multiplexed tokens.}
We introduce \ours{}, a local distillation method for continuous latent reasoning based on multiplexed targets (\Cref{fig:overview}). For each latent token, we define a vocabulary-space target by taking a position-weighted linear superposition of one-hot encodings in its corresponding discrete reasoning span. The model is trained to match this target through a linear-softmax head with a KL loss.

\item \textbf{Lossless multiplexing.}
We identify simple classes of positional weightings that guarantee lossless multiplexing, such that each superposed target fully preserves the discrete reasoning span it represents. These include geometric, sinusoidal, and rotary weightings, characterized by a subset-sum separation condition. We show that lossless multiplexing prevents shortcut behaviors found in prior methods caused by latent collapse.

\item \textbf{Parallel search via multiplexing.}
We show that, in problems requiring breadth-first search (BFS), multiplexed tokens are expressive enough to represent and update multiple hypotheses simultaneously, owing to their natively superposed construction, thereby implementing each BFS step using a single latent token. The result implies that parallel search can naturally emerge from serial supervision via multiplexing.

\item \textbf{Empirical results.}
\ours{} is the best latent reasoning method across 32 mathematical reasoning settings spanning two training corpora, four language models, and four test sets. It also surpasses strong discrete and continuous reasoning baselines on two search benchmarks. Through probing analysis, we show that the learned latent tokens encode interpretable reasoning content and contribute meaningfully to final prediction.

\end{enumerate}

\section{Related work}

We provide an overview of related work. An extended discussion can be found in \Cref{app:summary_comparison}.

\paragraph{Reasoning in language models.} Prior work has shown that language models benefit from making intermediate computations explicit. Early scratchpad methods \citep{nye2021show} showed that learning intermediate computation steps improves algorithmic problem solving. Chain-of-thought (CoT) prompting \citep{wei2022chain,kojima2022large} established natural language reasoning as a general mechanism for arithmetic, logical, and commonsense problem solving. Recent work has identified inefficiencies in language reasoning, showing that many reasoning tokens can be pruned \citep{li2026making, zhang2025lightthinker} or compressed \citep{xia2025tokenskip,li2026chain} with small degradation in accuracy. We share this motivation, but instead of shortening discrete reasoning at inference time, we distill their traces into compact continuous reasoning through training.

\paragraph{Reasoning in continuous latent spaces.}
A growing line of work explores reasoning in continuous latent spaces, broadly categorized into global and local methods.
Global methods supervise the answer token or trajectory endpoint and learn latent reasoning via temporal backpropagation, as in Coconut~\citep{hao2025training} and CODI~\citep{shen2025codi}. Local methods instead supervise each latent token to represent a span of discrete steps by aligning them in some choice of representation space, typically via auxiliary modules: SIM-CoT \citep{wei2026simcot} aligns in autoregressively decoded text space, KaVa~\citep{kuzina2026kava} in a key-value cache space. We question their complexity, and instead leverage linearly superposed representations in vocabulary space. While some prior work autoregress vocabulary-space vectors~\citep{zhang2026soft, deng2025latent, tang2026multiplex}, we use vocabulary projection only for supervision and autoregress directly in latent space at inference.

Theoretical works studied benefits of continuous latent reasoning, in particular superposition and parallel search \citep{gozeten2026continuous, zhu2026reasoning, wu2025parallel}. On the other hand, recent analyses caution that latent reasoning need not automatically encode faithful computations, sometimes acting as uninterpretable placeholders or exploiting shortcuts \citep{zhang2025latent}. \citet{cui2026latent} find pervasive shortcut behavior in global methods, and report that existing local methods mitigate shortcuts but trade off the ability to maintain diverse hypotheses in latent tokens. \citet{dilgren2026latent} propose vocabulary projection as a tool for interpreting latent tokens, arguing interpretability itself is a signal of reasoning correctness. Together with earlier probing studies \citep{cywinski2025towards, liang2026latent}, these motivate supervision methods whose targets are locally decodable, tied to explicit reasoning, and compatible with parallel search by design.

\section{\ours: Continuous reasoning via multiplexed tokens}
\label{sec:method}

\subsection{Problem setup}\label{sec:setup}

\paragraph{Language reasoning.}
Let $\Vocab$ be a discrete vocabulary of subwords and let $\Vocab^*$ be its associated text space. We denote text of length $L$ by ${\bf y} = ({\bf y}^1, ..., {\bf y}^L)$ with each ${\bf y}^l\in \Vocab$. Language models generate continuations of a text by autoregressively predicting the next subword.
While a language model may directly answer a given question ${\bf q}\mapsto \hat{\bf a}$ by continuation, prompting an intermediate reasoning ${\bf r}\in \Vocab^*$ before answering $({\bf q}, {\bf r})\mapsto \hat{\bf a}$ improves performance. This is, however, computationally inefficient.

\paragraph{Continuous latent reasoning.} To overcome the efficiency limitations of discrete reasoning, we reason in a choice of continuous vector space $X=\mathbb{R}^d$, where we denote by $X^*$ the set of vector sequences.
For each question ${\bf q}$, we would like to train a language model to articulate latent reasoning ${\bf x}\in X^*$ by autoregressing on continuous tokens ${\bf x}_1, {\bf x}_2, ... \in X$ before answering $({\bf q},{\bf x})\mapsto \hat{\bf a}$.
In practice, we treat $X^*$ as $X^K = \mathbb{R}^{K\times d}$ for a choice of $K$, which restricts each reasoning to a sequence of $K$ vectors.
Following prior work, we assume availability of triples $({\bf q},{\bf r},{\bf a})$ containing discrete reasoning traces ${\bf r}$, and use them to learn latent reasoning ${\bf x}$ via distillation.

\paragraph{Local distillation.} We focus on local distillation, where latent tokens are supervised with local spans of discrete reasoning steps. We assume each discrete trace ${\bf r}$ is chunked into spans $({\bf r}_1, ..., {\bf r}_M)$ where ${\bf r}_i = (r_i^1,\ldots,r_i^{S_i})\in \Vocab^{S_i}$. 
For example, in algorithmic and mathematical tasks, each span can be a step of computation, and in natural language, each span can be a sentence.
If a trace has more spans than latent tokens, $M>K$, some of the spans are merged heuristically (\Cref{app:alignment}); we thus assume $M \leq K$ onward. If $M < K$, some latent tokens have no aligned span. We let $\mathcal{K}$ denote the subset of latent token positions with nonempty span, noting $|\mathcal{K}|=M$.

In local distillation, latent reasoning ${\bf x}$ is trained so that each token ${\bf x}_i\in\mathbb{R}^d$ matches a span ${\bf r}_i\in \Vocab^*$ in some representation space $\mathcal Z$. This goal can be formalized as $f({\bf x}_i) = g({\bf r}_i)$ for some choice of maps $f:\mathbb{R}^d\to \mathcal Z$ and $g:\Vocab^*\to \mathcal Z$. The representation space and the maps constitute the core design decision of local distillation methods. SIM-CoT~\citep{wei2026simcot} uses $\mathcal Z=\Vocab^*$ with an autoregressive $f:\mathbb{R}^d\to \Vocab^*$ and $g={\rm id}$, and KaVa~\citep{kuzina2026kava} takes as $\mathcal Z$ the space of key-value cache and performs cache distillation. In contrast, we simply choose $\mathcal Z$ as the vocabulary simplex $\Delta^{|\Vocab|-1}$, or the space of $|\Vocab|$-dimensional probability vectors.

\subsection{Local distillation by multiplexing}\label{sec:local_distillation_by_multiplexing}

We now present our method for continuous latent reasoning via local distillation. To motivate it, we consider the case where $\mathcal Z$ is a fixed-dimensional vector space. Then, $g:\Vocab^*\to \mathcal Z$ can be viewed as an operator that combines a variable-dimensional categorical signal $i\mapsto {\bf r}_i$ into one, fixed-dimensional continuous signal $i\mapsto g({\bf r}_i)$, which defines the learning target $f^*({\bf x}_i)=g({\bf r}_i)$ for each latent token ${\bf x}_i$ via an optimal decoder $f^*$. Given this observation, it is natural to conceptualize $g$ as a type of \emph{multiplexed} encoding of variable-length categorical signal, $g=\mathsf{mux}$. We now identify the core requirement for local distillation based on multiplexing as follows.

\begin{definition}[\bf Multiplexing]
\label{def:mux}
A map $\mathsf{mux}:\Vocab^*\to \mathcal Z$ is \textit{spanwise injective} if $\mathsf{mux}|_{\Vocab^S}$ is injective for any span length $S\geq 1$. We say latent reasoning $({\bf x}_1, \ldots, {\bf x}_K)$ under an optimal decoder $f^*:\mathbb{R}^d\to \mathcal Z$ \textit{multiplexes} discrete reasoning spans $({\bf r}_1, \ldots, {\bf r}_M)$ if there exists a spanwise injective $\mathsf{mux}$ satisfying
\begin{equation}\label{eq:mux}
f^*({\bf x}_i) = \mathsf{mux}({\bf r}_i),\quad\forall i\in\mathcal{K}.
\end{equation}
\end{definition}

Equation~\eqref{eq:mux} requires that each latent token represents a local span of a discrete reasoning trace through multiplexing. As a training objective, it is used to drive $f({\bf x}_i)$ toward $f^*({\bf x}_i)= \mathsf{mux}({\bf r}_i)$ by jointly learning the latent reasoning ${\bf x}$ and the decoder $f$. Injectivity means that the representation $\mathsf{mux}$ is lossless, admitting an inverse (demultiplexing). Under lossless multiplexing, the local distillation target $\mathsf{mux}({\bf r}_i)$ is fixed-dimensional, allowing for scalable optimization, and encodes full information of each span ${\bf r}_i$, enabling faithful reasoning.

\begin{wrapfigure}{r}{0.44\columnwidth}
    \centering
    \includegraphics[width=\linewidth]{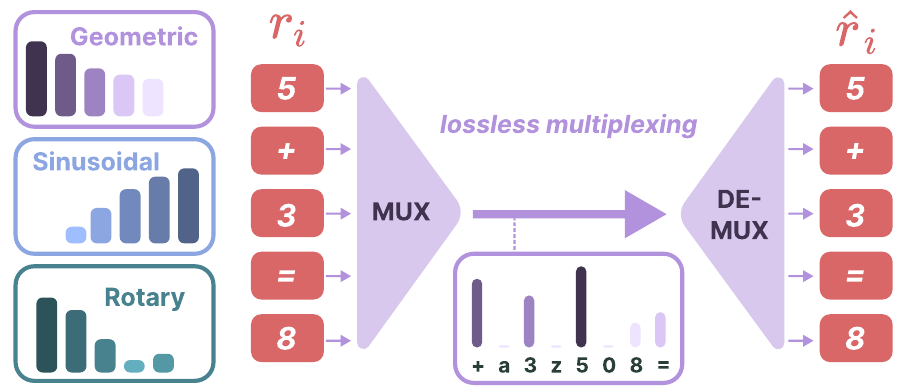}
    \vspace{-1.2em}
    \caption{Lossless multiplexing of a span \texttt{<<5+3=8>>} through position-weighted linear superposition.}
    \label{fig:mux}
    \vspace{-1.0em}
\end{wrapfigure}
\paragraph{Multiplexing via linear superposition.}
Constructing a spanwise lossless multiplexer is nontrivial, as it must handle variable-length categorical signals. Here, inspired by code-division schemes in communication systems, we propose a class of simple and training-free multiplexers based on linear superposition of one-hot encodings in the vocabulary space.
Concretely, for each discrete reasoning span ${\bf r}_i = (r_i^1, \ldots, r_i^{S_i})$, we define
\begin{equation}\label{eq:our_mux}
\muxop(\mathbf r_i) \coloneqq \sum_{j=1}^{S_i} \alpha_j^{(i)}\,\onehot(r_i^j),
\qquad
\alpha_j^{(i)}
\coloneqq
\frac{w_j}{\sum_{\ell=1}^{S_i} w_\ell},
\end{equation}
where $w_{[\cdot]}:\mathbb{N}\to\mathbb{R}_+$ is a choice of positional weighting.
Because the coefficients $\alpha_j^{(i)}$ are positive and normalized, $\muxop(\mathbf r_i)$ always lies in the vocabulary simplex $\Delta^{|\Vocab|-1}$.
To match such targets following \eqref{eq:mux}, we decode each latent token ${\bf x}_i$ through a linear-softmax head
\begin{equation}\label{eq:mux_decoder}
f(\mathbf x_i)\coloneqq \softmax(W\mathbf x_i/\tau),
\qquad W\in\R^{|\Vocab|\times d},
\end{equation}
where $W$ is the pretrained language model's unembedding layer and $\tau > 0$ is a temperature variable. At inference time, the model still autoregresses latent tokens ${\bf x}_i$ in hidden space, and the vocabulary projection is needed only for their supervision.

\paragraph{Positional weighting.} We propose three families of positional weightings $w_{[\cdot]}$ (\Cref{fig:mux}):
\begin{enumerate}[label=(\arabic*),leftmargin=*]
\item \emph{Geometric.} $w_j=\rho^{j-1}$ with decay rate $\rho\in(0,1)$. Earlier positions receive exponentially more weight, producing a monotonically decaying profile.
\item \emph{Sinusoidal.} $w_j=\exp(\lambda s_j)$ with scores $s_j = \sin(\tfrac{\pi}{2}\cdot\tfrac{j-1}{\max(S-1,1)})$ and scale $\lambda > 0$. This induces a monotonically increasing weighting that peaks near the end of a span.
\item \emph{Rotary.} $w_j=\exp(\lambda s_j)$ with scores $s_j = \tfrac{1}{P}\sum_{p=1}^{P}\cos(\theta_p(j-1))$, where $(\theta_p)_{p\leq P}$ is a set of positive frequencies analogous to rotary position embeddings \citep{su2024roformer}. Averaging cosine components across frequencies yields an expressive positional weighting.
\end{enumerate}
All of the above weightings yield spanwise lossless multiplexing when configured properly, as we show in \Cref{sec:lossless}. This outcome is not trivial. The sum $\mathsf{mux}({\bf r}_i)$ records only the \emph{total} mass of each unique subword in a span ${\bf r}_i$, so if a subword appears more than once, its positions in the span can be ambiguous. For example, with uniform weighting $\alpha_j=1/S$, the mass of a subword only counts how many times it appears, dropping the positions.

Therefore, the key to lossless multiplexing is to choose positional weights $\alpha_1,...,\alpha_S$ such that different sets of positions always produce different total masses.
Our theory in \Cref{sec:lossless} formalizes this as a subset-sum separation condition on the weights, satisfied by all of our weightings for proper hyperparameters.
Then, even if a subword appears multiple times, its total mass uniquely determines which positions it occupies, so the original span can be recovered exactly from its multiplexing.
In \Cref{sec:anticollapse}, we show that this property consequently prevents shortcut behaviors that are caused by the collapse of latent tokens.

\paragraph{Training objective.}
In order to train a language model to perform continuous latent reasoning, we use a composite loss
$\mathcal{L} = \mathcal{L}_{\mathrm{answer}} + \beta\,\mathcal{L}_{\mathrm{local}} + \gamma\,\mathcal{L}_{\mathrm{global}}$ with weights $\beta,\gamma\geq 0$, where each term corresponds to answer prediction loss, local distillation loss under multiplexed targets \eqref{eq:our_mux}, and an optional trajectory-level loss, detailed as follows.
The answer loss is standard cross entropy $\mathcal{L}_{\mathrm{answer}} = -\log p_\theta(\mathbf{a}\mid\mathbf{q}, \mathbf{x}_1,\ldots,\mathbf{x}_K)$, where $p_\theta$ is the likelihood evaluated by the language model.
For the local distillation loss, recall that the decoder $f:\mathbb{R}^d\to\Delta^{|\Vocab|-1}$ is a linear projection followed by tempered softmax \eqref{eq:mux_decoder}.
For each latent token ${\bf x}_{i\in\mathcal{K}}$ aligned with a nonempty span $\mathbf{r}_i$, we minimize the KL divergence between the model prediction $f({\bf x}_i)$ and the multiplexed target $\mathsf{mux}({\bf r}_i)$:
\begin{equation}\label{eq:local_loss}
\mathcal{L}_{\mathrm{local}} = \frac{1}{|\mathcal{K}|}\sum_{i\in\mathcal{K}} \mathrm{KL}\!\left(\mathsf{mux}(\mathbf{r}_i)\;\big\|\;f(\mathbf{x}_i)\right).
\end{equation}
Lastly, following \citet{shen2025codi}, we use an optional trajectory-level loss that aligns the hidden features at the answer token in the language model with continuous reasoning, with respect to those from a model with discrete reasoning,
trained with standard next-token prediction on discrete reasoning traces. We employ parameter sharing between the two models, which offers efficiency. The trajectory-level loss provides learning signal for the ``spare'' tokens $\mathbf{x}_{M+1},...,\mathbf{x}_K$ when $M<K$, which lack local targets. For the respective hidden features $\mathbf h_{\mathrm{answer}}^{\mathrm{cont}}$ and $\mathbf h_{\mathrm{answer}}^{\mathrm{disc}}$, we use $\mathcal L_{\mathrm{global}} = \|\mathbf h_{\mathrm{answer}}^{\mathrm{cont}}-\operatorname{sg}(\mathbf h_{\mathrm{answer}}^{\mathrm{disc}})\|_2^2$, where $\operatorname{sg}(\cdot)$ denotes the stop-gradient operator.
Together, the composite loss provides direct learning signal for every latent token as well as answer prediction.

\section{Theoretical analysis}
We organize the theory around the utility of multiplexing for latent reasoning. In \Cref{sec:lossless}, we ask when multiplexing is lossless,
and which positional weightings satisfy this criterion. Losslessness guarantees that every latent token encodes faithful computation without degrading into uninformative placeholders. In \Cref{sec:anticollapse}, we make this precise and show that multiplexing prevents latent collapse.
In \Cref{sec:parallel-search}, we prove that multiplexed tokens can implement parallel search by encoding an entire search frontier. All proofs are in \Cref{app:proofs}.

\label{sec:theory}
\subsection{Lossless multiplexing}\label{sec:lossless}

Consider multiplexing a discrete reasoning span ${\bf r}_i = (r_i^1, ..., r_i^S)$ into a continuous token $\mathsf{mux}({\bf r}_i)$ using normalized masses $\boldsymbol{\alpha}\!=\!(\alpha_1, ..., \alpha_{S})$ \eqref{eq:our_mux}. Each span ${\bf r}_i $ is a short sequence of reasoning tokens, and the weights $\alpha_j$ determine how much each position contributes to the resulting latent representation.
Our goal is to identify conditions that make $\mathsf{mux}$ spanwise lossless or injective (\Cref{def:mux}).
The following quantity will be central in our results:
\begin{definition}[\bf Subset-sum separation]
\label{def:E}
Let $\mathcal{C}_S$ be the set of all nonzero sequences ${\bf c}=(c_1, ..., c_S)$ taking values in $\{-1, 0, 1\}$ and consider the following measure of subset-sum collisions:
\begin{equation}\label{eq:subset_sum}
\Esep(\boldsymbol{\alpha}) \coloneqq \min_{{\bf c}\in \mathcal{C}_S} \left|\sum_{j=1}^{S} c_j \alpha_j\right|
\end{equation}
\end{definition}
Intuitively, $\Esep(\boldsymbol{\alpha})>0$ if and only if there are no distinct subsets of $\{1, ..., S\}$ having an identical total mass. We now characterize the exact criterion for lossless multiplexing as follows.

\begin{restatable}
[\bf Span-level lossless multiplexing]{proposition}{slot}
\label{thm:slot}
Assume $|\Vocab|>1$ and fix a span length $S$. Then the map $\muxop:\Vocab^S\to\Delta^{|\Vocab|-1}$ is injective if and only if $\Esep(\boldsymbol{\alpha})>0$.
\end{restatable}

The result shows that every finite span of discrete reasoning can be recovered exactly (demultiplexed) from its weighted linear superposition if the subset sums of the weights never collide. Based on this single-span case, we now consider an extension to full reasoning trace; the only additional ingredient is that a full reasoning trace is chunked into several spans, possibly of different lengths.

\begin{restatable}[\bf Trace-level lossless multiplexing]{corollary}{chain}
\label{cor:chain}
Let $\mathbf r=(\mathbf r_1,\ldots,\mathbf r_M)$ be a discrete reasoning trace where each ${\bf r}_i$ has length $S_i$ and normalized masses $\boldsymbol\alpha^{(i)}$. If $\Esep\bigl(\boldsymbol\alpha^{(i)}\bigr)>0$ for every $i$, then $\mathbf r$ is uniquely recoverable from the collection of multiplexed targets $\muxop(\mathbf r_i)$ together with $S_i$ and $\boldsymbol\alpha^{(i)}$.
\end{restatable}

We now identify the hyperparameter choices for our positional weightings (\Cref{sec:local_distillation_by_multiplexing}) that support lossless multiplexing. This can be characterized compactly: geometric weights admit an exact algebraic criterion, and the remaining exponential weights are injective if the scores are distinct.

\begin{restatable}[\bf Weightings for lossless multiplexing]{proposition}{proWeightings}
\label{pro:weightings}
\leavevmode
\begin{enumerate}[label=(\roman*),leftmargin=*]
    \item \emph{Geometric.} For $w_j=\rho^{j-1}$, multiplexing is injective iff $\rho$ is not a root of any nonzero polynomial $\sum_{j=1}^{S} c_j x^{j-1}$ with coefficients $c_j\in\{-1,0,1\}$. If $\rho\in(0,1)$ is rational, this holds for all finite $S$.
    \item \emph{Exponential.} For $w_j=\exp(\lambda s_j)$, if span length $S\ge 2$ and $s_1,...,s_S$ are pairwise distinct, then multiplexing is injective for all but finitely many values of $\lambda$.
\end{enumerate}
\end{restatable}

The sinusoidal and rotary weightings are special cases of the exponential $w_j=\exp({\lambda s_j})$, and so the above result implies that sinusoidal weighting is generally lossless. For rotary weightings, we prove a simple sufficient condition that all of its frequencies lie on the first decreasing branch of cosine.
Together, these results show that simple weightings can achieve spanwise lossless multiplexing.

\begin{restatable}[\bf Sinusoidal and rotary weightings for lossless multiplexing]{corollary}{corSinRot}
\label{cor:sin+rot}
\leavevmode
\begin{enumerate}[label=(\roman*)]
    \item For any $S\ge 2$, sinusoidal weighting yields an injective multiplexing for all but finitely many~$\lambda$.
    \item If \(0\!<\!\theta_p(S\!-\!1)\!<\!\pi\,\forall p\), then rotary weighting yields an injective multiplexing for all but finitely many~$\lambda$.
\end{enumerate}
\end{restatable}

\paragraph{Finite precision.}
The results above are under exact arithmetic. In practice, multiplexing is done in finite precision, so it is natural to ask whether our findings remain meaningful. Let \(\varepsilon_{\mathrm{fp}}\) be the worst-case error between multiplexed target and its finite-precision rounding. In \Cref{app:finite_precision}, we show that the separation margin \(\Esep(\boldsymbol\alpha)\) also governs numerical error: if \(\varepsilon_{\mathrm{fp}}<\Esep(\boldsymbol\alpha)/2\), then the original span remains exactly recoverable. Under the standard unit-roundoff model~\citep{goldberg1991every,higham2002accuracy}, \(\varepsilon_{\mathrm{fp}}\) admits an \(O(Su)\) bound for span length \(S\) and arithmetic precision \(u\). For our default geometric weighting \(\rho=0.9\) in \texttt{float32}, multiplexing is lossless for all span lengths faced in experiments (\(S\leq 11\)).

\subsection{Latent diversity}
\label{sec:anticollapse}

Intuitively, lossless multiplexing is beneficial as it enforces latent tokens to hold meaningful computation. We make this intuition precise and show that \ours{} guarantees diversity of latent tokens, avoiding semantic homogenization of latent reasoning shown by \citet{wei2026simcot} for global methods.

\begin{definition}[\bf Latent collapse]\label{def:collapse}
A continuous reasoning ${\bf x}$ exhibits \emph{collapse at level $\varepsilon\geq 0$} if
\[
\frac{1}{|\mathcal{K}|^2}\sum_{i,j\in\mathcal{K}} \|\mathbf{x}_i - \mathbf{x}_j\|_2^2 \leq \varepsilon.
\]
\end{definition}
\vspace{-0.3em}
\begin{definition}[\bf Target diversity]\label{def:target_diversity}
The \emph{target diversity} of discrete reasoning ${\bf r}$ under multiplexing is $\mathcal{D} \coloneqq \min_{\substack{i,j\in\mathcal{K},i\neq j}} \|\mathsf{mux}(\mathbf{r}_i) - \mathsf{mux}(\mathbf{r}_j)\|_1$.
\end{definition}
Whenever two discrete spans are distinct ${\bf r}_i\neq {\bf r}_j$ and the positional weighting is lossless $\Esep(\boldsymbol{\alpha})>0$, \Cref{thm:slot} guarantees $\mathsf{mux}({\bf r}_i)\neq\mathsf{mux}({\bf r}_j)$, so $\mathcal{D}>0$. We now show that local distillation on diverse targets forces diversity in latent tokens. This is empirically supported in \Cref{app:interp_analysis}.

\begin{restatable}[\bf Non-collapsing guarantee for multiplexed distillation]{proposition}{anticollapse}\label{thm:anticollapse}
Let $\widetilde W := W/\tau$ be the scaled readout matrix. Suppose $\mathcal{D}>0$, $\|\widetilde W\|_{\mathrm{op}}>0$, and
$\mathcal{L}_{\mathrm{local}}\le \delta < \mathcal{D}^2/8$.
Then
\begin{equation}\label{eq:anticollapse}
\frac{1}{|\mathcal{K}|^2}\sum_{i,j\in\mathcal{K}} \|{\bf x}_i - {\bf x}_j\|_2^2
\;\ge\;
\frac{|\mathcal{K}|-1}{|\mathcal{K}|}
\left(
\frac{\mathcal{D} - 2\sqrt{2\delta}}{\|\widetilde W\|_{\mathrm{op}}\,C_{|\Vocab|}}
\right)^2,
\end{equation}
where $C_{|\Vocab|}$ depends only on $|\Vocab|$.
Thus, latent tokens cannot collapse at any level below the right-hand side.
\end{restatable}

\subsection{Parallel search with multiplexed reasoning}
\label{sec:parallel-search}
We now consider search problems where each latent token has to represent a \emph{set} of hypotheses. In a graph reachability problem, there may be several nodes that have been explored and are waiting to be expanded. A continuous token can, in principle, carry such a set all at once in superposition, instead of forcing the model to commit to one possibility. We show that \ours{} preserves this advantage.

As a setup, consider the depth-$H$ reachability problem on a finite directed graph $G=(\mathcal N,E)$: given a source node $s\in\mathcal N$ and a target node $t\in\mathcal N$, the task is to determine whether there is a directed path $s \to t$ of length $\leq H$. A standard breadth-first search (BFS) maintains two sets at each step $k$: the frontier node set $F_k$ discovered for the first time, and the node set $U_k$ discovered so far.
Denoting by $N^+(B)$ the out-neighborhood of a node set $B$, each BFS step updates, from $F_0=\{s\}, U_0=\{s\}$:
\[
F_{k+1}=N^+(F_k)\setminus U_k,
\qquad
U_{k+1}=U_k\cup F_{k+1},
\qquad
k=0,\dots,H-1.
\]
Suppose the discrete reasoning at step $k$ is $\mathbf r_k=(r_k^1,\dots,r_k^{|F_k|})$ that lists the elements of $F_k$ in an arbitrary order. In this setting, the object of interest is which nodes are in $F_k$, which can be fully encoded with multiplexing $\mathsf{mux}(\mathbf r_k) = \frac{1}{|F_k|} \sum_{j=1}^{|F_k|} \mathsf{onehot}(r_k^j)$ as a uniform distribution over $F_k$.
We now prove that this target is expressive enough to carry and expand an entire frontier $F_k$ together with the discovered set $U_k$, thus implementing BFS.
\begin{restatable}[\bf Parallel BFS with multiplexing]{proposition}{parallel}
\label{thm:bfs}
There exists a sequence of continuous tokens $({\bf x}_0, \ldots, {\bf x}_H)$ such that, for every $k\le H$:
\begin{enumerate}[label=(\roman*),leftmargin=*]
    \item ${\bf x}_k$ is a deterministic function of ${\bf x}_{k-1}$ and $G$,
    \item $F_k$ and $U_k$ can be recovered from ${\bf x}_k$, and so reachability $\boldsymbol{1}(t\in U_H)$ can be recovered from~${\bf x}_H$,
    \item whenever $F_k\neq\varnothing$, $\mathsf{mux}(\mathbf r_k)$ can be recovered from ${\bf x}_k$, up to arbitrary precision with softmax.
\end{enumerate}
\end{restatable}
The result implies that parallel search can naturally emerge from serial supervision via multiplexing.

\section{Experiments}
We evaluate \ours{} on mathematical reasoning (\Cref{sec:mathematical_reasoning_results}), verify its parallel search capabilities (\Cref{sec:parallel_search}), and analyze the role of key design choices (\Cref{sec:ablations}). Interpretability and attention analysis can be found in \Cref{app:interp_analysis,app:attention_analysis}, and training cost analysis can be found in \Cref{app:training_cost}.

\label{sec:experiments}
\subsection{Mathematical reasoning}\label{sec:mathematical_reasoning_results}

\begin{table}[t]
\centering
\caption{Mathematical reasoning test accuracies (\%). $^{\dagger}$ and $^{\ddagger}$ are from \citet{shen2025codi} and \citet{kuzina2026kava}, respectively. We underline \ours{} when it outperforms SFT-CoT. \ours{} reports $\pm$1 std.\ over 3 seeds.
We did not conduct iCoT/Coconut OOD tests on NL due to their low ID scores.}\label{tab:main_results}
\scriptsize
\resizebox{\textwidth}{!}{%
\begin{tabular}{lllll|llll}
\toprule
\multirow{2}{*}{\textbf{Method}}
& \multicolumn{4}{c|}{\textbf{GSM8K-AUG}}
& \multicolumn{4}{c}{\textbf{GSM8K-AUG-NL}} \\
\cmidrule(lr){2-5}\cmidrule(lr){6-9}
& \textbf{ID} & \textbf{SVAMP} & \textbf{GSM-Hard} & \textbf{MultiArith}
& \textbf{ID} & \textbf{SVAMP} & \textbf{GSM-Hard} & \textbf{MultiArith} \\
\midrule
& \multicolumn{8}{c}{\textbf{GPT-2}} \\
\midrule
\textcolor{gray}{SFT-CoT} & \textcolor{gray}{44.1$^{\dagger}$} & \textcolor{gray}{41.8$^{\dagger}$} & \textcolor{gray}{9.8$^{\dagger}$} & \textcolor{gray}{90.7$^{\dagger}$} & \textcolor{gray}{34.2} & \textcolor{gray}{36.9} & \textcolor{gray}{7.1} & \textcolor{gray}{88.7} \\
No-CoT$^{\dagger}$ & 19.1 & 16.4 & 4.3 & 41.1 & 19.1 & 16.4 & 4.3 & 41.1 \\
\cmidrule(lr){1-9}
\multicolumn{4}{l}{\emph{Latent reasoning}}\\
iCoT & 30.1$^{\dagger}$ & 29.4$^{\dagger}$ & 5.7$^{\dagger}$ & 55.5$^{\dagger}$ & 3.2  & -- & -- & -- \\
Coconut & 34.1$^{\dagger}$ & 36.4$^{\dagger}$ & 7.9$^{\dagger}$ & 82.2$^{\dagger}$ & 24.9 & --   & --   & --   \\
CODI & {43.7} & {42.9} & {9.9} & {92.8} & {34.1} & {30.8} & 6.8 & {58.9} \\
SIM-CoT & 42.6 & 42.6 & 9.4 & {92.8} & 30.9 & 27.5 & 6.5 & 53.9 \\
\rowcolor{OursBg}
\ours{} & \underline{\textbf{48.1}} \scriptsize{$\pm$0.3} & \underline{\textbf{45.0}} \scriptsize{$\pm$0.7} & \underline{\textbf{10.6}} \scriptsize{$\pm$0.5} & \underline{\textbf{93.0}} \scriptsize{$\pm$0.8} & \underline{\textbf{37.4}} \scriptsize{$\pm$0.2} & \textbf{36.7} \scriptsize{$\pm$0.7} & \underline{\textbf{8.9}} \scriptsize{$\pm$0.4} & \textbf{72.4} \scriptsize{$\pm$1.6} \\
\midrule
& \multicolumn{8}{c}{\textbf{LLaMA 3.2 1B-Instruct}} \\
\midrule
\textcolor{gray}{SFT-CoT} & \textcolor{gray}{61.6$^{\dagger}$} & \textcolor{gray}{66.7$^{\dagger}$} & \textcolor{gray}{15.6$^{\dagger}$} & \textcolor{gray}{99.3$^{\dagger}$} & \textcolor{gray}{53.2} & \textcolor{gray}{62.9} & \textcolor{gray}{13.3} & \textcolor{gray}{98.5} \\
No-CoT$^{\dagger}$ & 30.9 & 44.1 & 7.1 & 70.9 & 30.9 & 44.1 & 7.1 & 70.9 \\
\cmidrule(lr){1-9}
\multicolumn{4}{l}{\emph{Latent reasoning}}\\
iCoT & 19.0$^{\dagger}$ & 40.9$^{\dagger}$ & 4.4$^{\dagger}$ & 39.0$^{\dagger}$ & $15.2$ & -- & -- & -- \\
Coconut & 45.3$^{\dagger}$ & 48.8$^{\dagger}$ & 9.9$^{\dagger}$ & 90.1$^{\dagger}$ & 24.2 & --   & --   & --   \\
CODI & 55.6 & 61.1 & {12.8} & 96.1 & {47.9} & {55.3} & {11.3} & {96.7} \\
SIM-CoT & {56.1} & {61.5} & 12.7 & 96.2 & 28.4 & 43.0 & 6.6 & 59.4 \\
\rowcolor{OursBg}
\ours{} & \textbf{56.7} \scriptsize{$\pm$0.5} & \textbf{63.6} \scriptsize{$\pm$1.0} & \textbf{13.0} \scriptsize{$\pm$0.2} & \textbf{98.5} \scriptsize{$\pm$0.9} & \textbf{50.3} \scriptsize{$\pm$0.3} & \textbf{57.5} \scriptsize{$\pm$0.6} & \textbf{11.6} \scriptsize{$\pm$0.2} & \textbf{96.9} \scriptsize{$\pm$0.6} \\
\cmidrule(lr){1-9}
\multicolumn{4}{l}{\emph{Latent reasoning via Jacobi iterations}}\\
PCCoT & 53.5 & 57.6 & \textbf{12.9} & {97.2} & 50.1 & 54.6 & 12.2 & {96.8} \\
KaVa$^{\ddagger}$ & {56.5} & {58.9} & {12.7} & --  & {55.7} & {58.6} & {12.8} & -- \\
\rowcolor{OursBg}
\ours{} & \textbf{58.0} \scriptsize{$\pm$0.5} & \textbf{61.8} \scriptsize{$\pm$0.5} & \textbf{12.9} \scriptsize{$\pm$0.4} & \textbf{98.7} \scriptsize{$\pm$0.7} & \underline{\textbf{57.2}} \scriptsize{$\pm$0.6} & \textbf{60.6} \scriptsize{$\pm$2.0} & \underline{\textbf{13.4}} \scriptsize{$\pm$0.5} & \underline{\textbf{99.2}} \scriptsize{$\pm$0.3} \\
\bottomrule
\end{tabular}}
\end{table}

\begin{table}[t]
\centering
\caption{Scaling to larger backbones on GSM8K-AUG (\%). $^{\diamond}$ results from \citet{wei2026simcot}. We underline \ours{} when it outperforms SFT-CoT. Single runs due to resource limits.}\label{tab:scaling_results}
\footnotesize
\begin{tabular}{lllll|llll}
\toprule
\multirow{2}{*}{\textbf{Method}}
& \multicolumn{4}{c|}{\textbf{LLaMA 3.2 3B}}
& \multicolumn{4}{c}{\textbf{LLaMA 3.1 8B}} \\
\cmidrule(lr){2-5}\cmidrule(lr){6-9}
& \textbf{ID} & \textbf{SVAMP} & \textbf{GSM-Hard} & \textbf{MultiArith}
& \textbf{ID} & \textbf{SVAMP} & \textbf{GSM-Hard} & \textbf{MultiArith} \\
\midrule
\textcolor{gray}{SFT-CoT}$^{\diamond}$ & \textcolor{gray}{71.5} & \textcolor{gray}{71.0} & \textcolor{gray}{17.0} & \textcolor{gray}{98.3} & \textcolor{gray}{71.7} & \textcolor{gray}{73.1} & \textcolor{gray}{16.5} & \textcolor{gray}{98.3} \\
No-CoT$^{\diamond}$ & 38.3 & 52.9 & 9.5 & 88.7 & 39.5 & 55.3 & 9.8 & 88.0 \\
CODI$^{\diamond}$ & 60.8 & 73.3 & 14.3 & 98.7 & 61.1 & 78.1 & 15.5 & {99.5} \\
SIM-CoT & {62.3} & {74.9} & {14.6} & {98.8} & {64.1} & {79.4} & {16.3} & \textbf{100.0} \\
\rowcolor{OursBg}\ours & \textbf{65.0} & \underline{\textbf{77.1}} & \textbf{15.2} & \underline{\textbf{100.0}} & \textbf{68.1} & \underline{\textbf{80.1}} & \underline{\textbf{17.1}} & \underline{\textbf{100.0}} \\
\bottomrule
\end{tabular}%
\end{table}

\paragraph{Setup.}
We follow the protocol of prior work and, for training, use two reasoning-augmented mathematical corpora built upon GSM8K~\citep{cobbe2021training}. GSM8K-AUG~\citep{shen2025codi} includes structured reasoning from GPT-4~\citep{achiam2023gpt}, whereas GSM8K-AUG-NL~\citep{deng2023implicit} includes informal linguistic reasoning. We use four test sets: GSM8K test split which is in-domain, and out-of-domain arithmetic datasets SVAMP~\citep{patel2021nlp}, GSM-Hard~\citep{gao2023pal}, and MultiArith~\citep{roy2015solving} to test for transferability under distribution shift. We mainly use GPT-2~\citep{radford2019language} and LLaMA~3.2~1B-Instruct~\citep{llama32} as backbone language models for \ours{} and baselines, and post-train them via LoRA~\citep{hu2022lora}. To assess scalability, we also test larger backbones LLaMA~3.2~3B and 3.1~8B on GSM8K-AUG following the protocol of \citet{wei2026simcot}; we were unable to train them on GSM8K-AUG-NL due to resource limits, as reasoning traces therein are considerably longer. For \ours{} and baselines, we mainly follow the setup of \citet{shen2025codi}, generating six latent tokens sequentially. For improved scalability, we also experiment with the setup of \citet{wu2025parallel} where 24 latent tokens are generated in parallel via three Jacobi iterations~\citep{ortega2000iterative}, using LLaMA~3.2~1B-Instruct as backbone.
CommonsenseQA~\citep{talmor2019commonsenseqa} and StrategyQA~\citep{geva2021did}, which are non-mathematical, have been tested in prior work~\citep{shen2025codi,wu2025parallel,wei2026simcot}, but are known to produce high-variance, unreliable results for latent reasoning methods~\citep{shen2025codi,wu2025parallel}. We therefore omit them.

We compare against non-reasoning, discrete-reasoning, and latent-reasoning baselines. SFT-CoT is supervised on discrete reasoning; No-CoT predicts only the answer; iCoT~\citep{deng2023implicit} internalizes discrete reasoning into a forward pass; Coconut~\citep{hao2025training} and CODI~\citep{shen2025codi} rely on trajectory-level losses for latent reasoning; SIM-CoT~\citep{wei2026simcot} adds local distillation via an autoregressive decoder. In the parallel decoding setting, we test latent methods PCCoT~\citep{wu2025parallel} and KaVa~\citep{kuzina2026kava} which are developed in the setting.

\paragraph{Results.} \Cref{tab:main_results,tab:scaling_results} show the results. \ours{} achieves the best latent reasoning performance in all 32 settings, surpassing both global (iCoT, Coconut, CODI, PCCoT) and local (SIM-CoT, KaVa) distillation methods for latent reasoning often by a large margin. Strikingly, \ours{} even outperforms discrete-reasoning SFT-CoT in 15 cases spanning all model scales and both in-domain and out-of-domain evaluations. This result is surprising since it shows that \ours{} is able to outperform the target of distillation, in a computationally efficient manner since generating six latent reasoning tokens corresponds to roughly $2.4\times$ and $5.9\times$ fewer reasoning tokens than SFT-CoT on GSM8K-AUG and GSM8K-AUG-NL, respectively. We conjecture that multiplexing for local distillation regularizes the language models to exhibit good generalization behaviors, while acquiring efficiency via compact superposed reasoning. Overall, the results suggest that \ours{} is a simple method that learns strong, generalizable, and efficient latent reasoning that scales with language model sizes. We further disentangle the effect of local supervision from that of global distillation in \Cref{app:local_distillation}, where the multiplexed target with $\gamma{=}0$ consistently surpasses SIM-CoT under the same regime.

\FloatBarrier
\subsection{Parallel search}
\label{sec:parallel_search}

\paragraph{Setup.} We evaluate \ours{} on tasks that require search, aiming to verify our theoretical results in \Cref{sec:parallel-search}. We consider two benchmarks, each naturally cast as a depth-$H$ reachability problem on a finite directed graph so that the BFS frontier and discovered set are well defined.
We set the number of latent tokens equal to graph depth.
In this setting, a sequential strategy tracking a single hypothesis per token cannot explore all states within this budget. Therefore, accuracy gains reflect the model's ability to represent and update multiple candidates in parallel.
Further details are in \Cref{app:search_details}.

\newcommand{\acc}[2]{$#1{\scriptstyle \pm #2}$}
\newcommand{\accbf}[2]{$\textbf{#1}{\scriptstyle \pm #2}$}
\begin{wraptable}{r}{0.30\columnwidth}
\centering
\vspace{-2em}
\caption{Search accuracies (\%).}
\vspace{-0.5em}
\label{tab:search_results}
\small
\begin{tabular}{lcc}
\toprule
\textbf{Method} & \textbf{MNNS} & \textbf{Game24} \\
\midrule
No-CoT & 68.4 & \acc{74.4}{2.1}  \\
SFT-CoT & \acc{84.6}{2.1} & \acc{84.3}{1.5} \\
Coconut & \acc{92.8}{0.6} & \acc{78.6}{2.0} \\
CoT$2$ & \acc{98.9}{0.3} & \acc{85.0}{1.5} \\
\midrule
\rowcolor{OursBg}
\ours{} & \accbf{99.6}{0.3} & \accbf{88.7}{1.1} \\
\bottomrule
\end{tabular}
\vspace{-1em}
\end{wraptable}
\paragraph{MNNS.} The minimum nonnegative sum (MNNS) task~\citep{gozeten2026continuous} asks, given a set of integers $a_1,\ldots,a_H$, for the smallest nonnegative value of $\sigma_1 a_1 + ... + \sigma_H a_H$ over signs $\sigma_k = \pm 1$. This can be viewed as a search problem over a directed graph $G=(\mathcal{N},E)$ whose node set is $\mathcal{N}=\{(k,z): k\in\{0,\ldots,H\},\; z\text{ a reachable partial sum}\}$, with edges $((k,z),(k{+}1,z\pm a_{k+1}))\in E$.
The source is $s=(0,0)$ and the answer is the minimum nonnegative $z$ such that $(H,z)\in U_H$.
At each depth $k$, the frontier $F_k$ contains all partial-sum states discovered for the first time.

\paragraph{Game of 24.} We introduce a new arithmetic search benchmark based on the Game of 24~\citep{yao2023tree}. Given $C$ cards drawn from $\{1,\ldots,D\}$ and an operator set $\mathcal{O}\subseteq\{+,-,\times\}$, the task is to determine whether the value 24 is reachable by folding the cards left to right, accumulating with an operator from $\mathcal{O}$. This defines a layered directed graph $G=(\mathcal{N},E)$ whose nodes at depth $k$ are all reachable values after accumulating the $k$-th card, and edges correspond to the available operations. The task is binary reachability: $y=\boldsymbol{1}(24\in U_{C-1})$. We use $C=5$ cards, digits $\{1,\ldots,5\}$, and $\mathcal{O}=\{+,-,\times\}$.

\paragraph{Results.} \Cref{tab:search_results} shows the results, averaged over 3 seeds. \ours{} achieves the best search performance in both tasks, directly verifying our claims in \Cref{sec:parallel-search} that multiplexed supervision can give rise to latent reasoning that performs parallel search. This supports that \ours{} is capable of exploring multiple hypotheses in superposition, a core advantage of latent reasoning methods.

\FloatBarrier
\subsection{Ablation studies}
\label{sec:ablations}
\paragraph{Contribution of local loss.} \ours{} learns from local and global distillation losses $\beta\,\mathcal{L}_{\mathrm{local}} + \gamma\,\mathcal{L}_{\mathrm{global}}$ with $\beta,\gamma\geq 0$ (\Cref{sec:local_distillation_by_multiplexing}). To isolate the role of local distillation via multiplexing, we set $\gamma=0$, and compare against SIM-CoT, which also uses loss $\beta\,\mathcal{L}_{\mathrm{local}}' + \gamma\mathcal{L}_{\mathrm{global}}$ where $\mathcal{L}_{\mathrm{local}}'$ learns local distillation in autoregressively decoded text space, by setting its $\gamma=0$. This allows comparing the quality of local loss in a controlled manner. \Cref{tab:ablation} shows that \ours{} indeed has a higher-quality local loss, although being simpler and not requiring an auxiliary autoregressive decoder.
A more comprehensive comparison can be found in \Cref{app:local_distillation}.

As a complementary analysis, we take \ours{} and vary how many latent tokens receive local loss. \Cref{fig:ablation_num_supervised_latents} shows accuracy rises from 32.7\% with no distilled tokens to 48.2\% with six. The model without distillation still performs latent reasoning, but no tokens are matched with discrete reasoning. This result shows that the gains of \ours{} come from local distillation, not merely from effective depth of latent reasoning.

\paragraph{Chunking strategy.} When $M > K$, the $M$ discrete spans must be merged into $K$ spans before being distilled into $K$ latent tokens (\Cref{sec:setup}). We compare randomized chunking, fixed deterministic chunking, and no chunking (truncation). \Cref{tab:ablation} shows that randomized chunking performs best. We attribute this to its resampled boundaries, which act as a form of structured data augmentation.

\begin{figure}[t!]
\centering
\begin{minipage}[t!]{0.48\textwidth}
\centering
\captionof{table}{Ablations (LLaMA 1B; GSM8K-AUG).}
\label{tab:ablation}
\scriptsize
\setlength{\tabcolsep}{4pt}
\begin{tabular}{lcccc}
\toprule
\textbf{Method} & \textbf{ID} & \textbf{SVAMP} & \textbf{GSM-Hard} & \textbf{MultiArith} \\
\midrule
\multicolumn{5}{l}{\textit{Local distillation only ($\gamma=0$)}} \\
SIM-CoT & 31.6 & 44.0 & 7.5 & 69.5 \\
\rowcolor{OursBg}
\ours{} & \textbf{48.7} & \textbf{51.2} & \textbf{10.6} & \textbf{98.9}  \\
\midrule
\multicolumn{5}{l}{\textit{Chunking strategy}} \\
None          & 54.3          & \textbf{61.3} & 12.5          & 96.5 \\
Deterministic & 55.7          & 60.6          & 12.6          & 97.3 \\
\rowcolor{OursBg}
Random        & \textbf{56.6} & \textbf{61.3} & \textbf{13.0} & \textbf{98.3} \\
\midrule
\multicolumn{5}{l}{\textit{Positional weighting}} \\
Uniform   & 54.2          & 62.4          & 12.8          & 96.1 \\
\rowcolor{OursBg}
Rotary    & 55.0          & \textbf{64.4} & 12.4          & 98.3 \\
\rowcolor{OursBg}
Sinusoidal & 55.2         & 61.1          & 12.6          & \textbf{99.4} \\
\rowcolor{OursBg}
Geometric & \textbf{57.2} & 63.1          & \textbf{12.9} & 98.3 \\
\bottomrule
\end{tabular}
\end{minipage}\hfill
\begin{minipage}[t!]{0.51\textwidth}
\centering
\includegraphics[width=\linewidth,height=125.18pt,keepaspectratio]{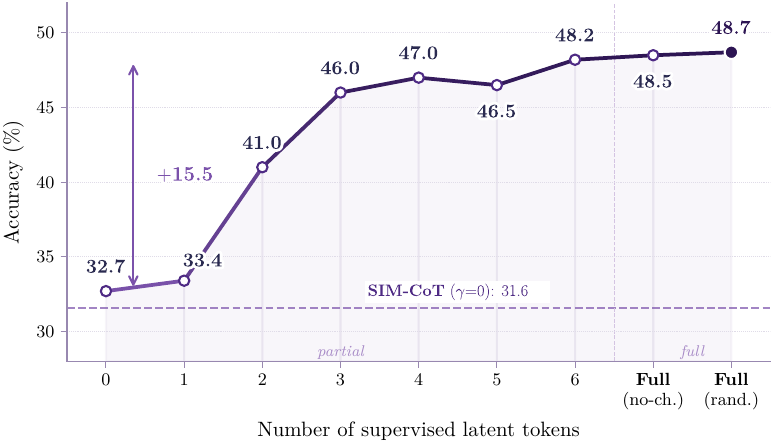}
\caption{Effect of local distillation loss.}
\label{fig:ablation_num_supervised_latents}
\end{minipage}
\end{figure}

\begin{wrapfigure}{r}{4cm}
\vspace{-1em}
\centering
\includegraphics[width=0.9\linewidth]{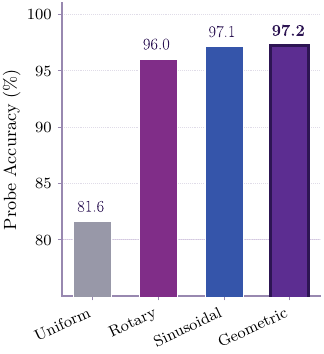}
\caption{Probe accuracy}
\vspace{-1em}
\label{fig:ablation_probe_weighting}
\end{wrapfigure}
\paragraph{Positional weighting.} We test the role of positional weighting, which theoretically affects multiplexing losslessness (\Cref{sec:local_distillation_by_multiplexing,sec:lossless}). We compare our weightings, proven to be lossless, against (lossy) uniform weighting. \Cref{tab:ablation} shows that while lossless weighting is better, the gap is modest, which could be attributed to the fact that many reasoning spans in GSM8K-AUG are highly structured, e.g., \texttt{<<60/2 = 30>>}, so ordering of subwords can often be inferred even from bags of subwords. Nevertheless, the overall gains suggest that losslessness is beneficial in practice.
We further train a small MLP to demultiplex spans \({\bf r}_i\) from \(\mathsf{mux}({\bf r}_i)\). \Cref{fig:ablation_probe_weighting} shows that uniform weighting allows nontrivial accuracy, confirming that ordering of subwords is partially recoverable from occurrences. Lossless weightings are still better, agreeing with latent reasoning performances.

\section{Conclusion}
We introduced \ours{}, a simple local distillation method for continuous latent reasoning based on position-weighted superposition in vocabulary space. Each latent token is trained to represent an aligned span of discrete reasoning via a multiplexed target that is easy to compute, theoretically grounded, and empirically effective. We showed that suitable positional weightings support exact span recovery, and that multiplexed targets can express parallel search dynamics. Across multiple models and benchmarks, \ours{} consistently improved upon strong baselines. These results suggest that simple, interpretable local targets can make latent reasoning stronger and easier to train.

\section*{Acknowledgments}
The authors would like to thank Xingyue Huang, Louis Tichelman, and Angelo Gnazzo for valuable discussions.

\bibliographystyle{plainnat}
\bibliography{references}

\beginappendix

\newpage
\section{Extended related work}
\label{app:related}

We expand the related-work discussion from the main text and then summarize the main distinctions in \Cref{tab:comparison}.

\subsection{Reasoning in language}
Chain-of-thought (CoT) prompting~\citep{wei2022chain, kojima2022large} showed that asking a language model to articulate intermediate reasoning steps dramatically improves performance on arithmetic, symbolic, and commonsense tasks. \citet{nye2021show} introduced scratchpads as a training-time analogue, where intermediate tokens serve as an explicit computation buffer.
Subsequent methods refine how this buffer is generated, verified, or searched, including self-consistency~\citep{wang2023selfconsistency}, STaR~\citep{zelikman2022star}, least-to-most prompting~\citep{zhou2023leasttomost}, Tree of Thoughts~\citep{yao2023tree}, and PAL~\citep{gao2023pal}.

At the same time, recent work has shown that such traces are often far more
verbose than what the underlying computation requires. TokenSkip~\citep{xia2025tokenskip},
step-entropy pruning~\citep{li2026making}, LightThinker~\citep{zhang2025lightthinker},
and ALiCoT~\citep{li2026chain} all indicate that explicit CoTs contain
redundant linguistic overhead. These findings motivate \ours{}. Our goal is not to compress a generated reasoning at inference time, but to use the redundancy of discrete traces to train a smaller number of continuous reasoning states.

\subsection{Implicit reasoning and internalization}

A related line of work tries to keep the benefits of intermediate computation while removing the need to emit intermediate language at inference time. iCoT~\citep{deng2023implicit}
and its stepwise extension~\citep{deng2024explicit} distill explicit reasoning
into computations within a single forward pass. Pause tokens~\citep{goyal2024think}, Quiet-STaR~\citep{zelikman2024quietstar},
Fast Quiet-STaR~\citep{huang-etal-2025-fast}, and thinking
tokens~\citep{herel2024thinking} increase internal compute by inserting special
positions that need not correspond to normal language. Compressed
CoT~\citep{cheng2024compressed} similarly move toward denser
reasoning representations. These methods increase effective compute depth without proportionally increasing output length.

\subsection{Continuous latent reasoning}

Continuous reasoning methods go further by operating in a latent vector space
and feeding latent tokens back to the model. We organize the literature by the type of supervision employed.

\paragraph{Global supervision.}
Coconut~\citep{hao2025training} is the foundational method, replacing discrete reasoning with latent recurrence, forming a ``chain of continuous thought.'' Its training uses a curriculum that
gradually transitions from discrete reasoning to fully latent reasoning.
Coconut demonstrated that continuous reasoning can support breadth-first-style exploration, but intermediate states are trained only from the final answer loss, which leaves the intermediate latent trajectory unsupervised.
CODI~\citep{shen2025codi} strengthened this with self-distillation, aligning the continuous and discrete reasoning modes in terms of the hidden state used to predict the final answer. Both methods supervise the reasoning process mainly through the final answer or trajectory endpoint. In our terminology, they are \emph{global} supervision methods that do not supervise what each latent token should represent. Recent empirical analyses show that this lack of intermediate supervision typically leads to shortcut behavior, as globally supervised models can achieve high accuracy without meaningfully relying on the latent reasoning tokens~\citep{cui2026latent, dilgren2026latent}.

\paragraph{Local supervision via auxiliary components.} \ours{} is closer to \emph{local} supervision methods, which supervise each latent reasoning token with a choice of target. 
SIM-CoT~\citep{wei2026simcot} identifies a critical limitation of global
supervision: as the number of latent tokens increases, they become homogeneous and training collapses. To address this, SIM-CoT uses an auxiliary autoregressive decoder during training that forces each latent token to encode its aligned discrete reasoning span, providing local supervision.
KaVa~\citep{kuzina2026kava} takes a different approach by distilling the teacher's compressed key-value (KV) cache into the student model layer by layer. The supervision target is the teacher's cache dynamics, providing a rich but structurally complex signal.
Both show that local supervision is effective in mitigating the failure mode of global supervision methods, but each requires additional components, which are an auxiliary decoder or a KV compression module.

\paragraph{Parallelization and efficiency.}
PCCoT~\citep{wu2025parallel} improves the efficiency of continuous reasoning by parallelizing sequential predictions of latent tokens via Jacobi iterations, reducing inference latency while maintaining accuracy.
SoftCoT~\citep{xu2025softcot} generates soft reasoning tokens from a frozen model using a trained projection layer, and SoftCoT++~\citep{xu2025softcot++} extends this to test-time compute scaling.
Token Assorted~\citep{su2025token} mixes discrete and continuous tokens in a hybrid reasoning trace, allowing the model to choose when to reason in language and when to reason in latent space.
TWT~\citep{xu-etal-2025-twt} distills reasoning from multiple teacher models into habitual latent computation.

\paragraph{Inference-time continuous reasoning.}
Another related line of work considers reasoning in a continuous space only at
\emph{inference time} by modifying the decoding procedure of a pretrained language model.
Soft Thinking~\citep{zhang2026soft} replaces discrete subword selection with
probability-weighted mixtures of vocabulary embeddings. Subsequent work studies
its limitations and variants~\citep{wu2026llms, tang2026multiplex, zheng2025soft}. Multiplex Thinking \citep{tang2026multiplex} samples a set of subwords at each reasoning step and aggregates their embeddings into a single continuous \emph{multiplex token}, maintaining vocabulary embedding priors while enabling on-policy RL. These methods are complementary to us. They modify the decoding procedure of a pretrained model, whereas \ours{} is a training-time method for latent reasoning distillation.  This separation lets us use vocabulary space for interpretable supervision without necessarily committing to it during inference time.

\subsection{Theoretical foundations of continuous reasoning}
A growing theoretical literature formalizes the advantages of continuous over discrete reasoning. \citet{zhu2026reasoning} prove that a two-layer transformer with $\mathrm{diameter}(G)$ steps of continuous reasoning can solve directed graph reachability on graph $G$.  The key mechanism is \emph{superposition}: each continuous thought vector encodes multiple search frontiers simultaneously, enabling parallel
breadth-first search.  In contrast, discrete reasoning requires $O(|V(G)|^2)$ steps with constant-depth transformers. CoT$2$~\citep{gozeten2026continuous} provides complementary results for search
problems, showing that supervision against latent token distributions induces parallel exploration. Our work connects to this line in two ways. We prove that our multiplexed targets are \emph{lossless} under standard positional weightings, and that a latent recurrence over such targets can implement exact parallel breadth-first exploration.

\subsection{Positioning of \ours{}}
\label{app:summary_comparison}

\begin{table}[t!]
\centering
\caption{Qualitative comparison of reasoning methods. \cmark\,=\,favorable, \xmark\,=\,unfavorable.}
\label{tab:comparison}
\small
\setlength{\tabcolsep}{4.5pt}
\renewcommand{\arraystretch}{1.15}
\begin{tabular}{l l ccccc}
\toprule
\textbf{Method} & \textbf{Supervision} & \textbf{Lossless} & \textbf{Shortcut-free} & \textbf{Train eff.} & \textbf{Infer.\ eff.} & \textbf{Interpretable} \\
\midrule
SFT-CoT & Discrete & \cmark & \cmark & \cmark & \xmark & \cmark \\
CODI & Global & \xmark & \xmark & \cmark & \cmark & \xmark \\
SIM-CoT & Local & \cmark & \cmark & \xmark & \cmark & \cmark \\
KaVa & Local & \xmark & \cmark & \cmark & \cmark & \cmark \\
\rowcolor{OursBg}
\ours{} & Local & \cmark & \cmark & \cmark & \cmark & \cmark \\
\bottomrule
\end{tabular}
\end{table}

\Cref{tab:comparison} summarizes the positioning of \ours{} relative to representative baselines along five desirable properties: whether the training signal preserves the full discrete reasoning trace (\emph{lossless}), whether latent tokens avoid collapsing into uninformative placeholders (\emph{shortcut-free}), whether training and inference are efficient (\emph{train/infer.\ eff.}), and whether intermediate latent states can be decoded into human-readable content (\emph{interpretable}).

SFT-CoT directly supervises every discrete reasoning subword via cross-entropy, so losslessness and shortcut avoidance hold by construction.
Training is efficient as it processes only the discrete reasoning sequence, but inference requires generating the full trace, which is $2.4$--$5.9\times$ more subwords than the compact budgets used by continuous methods (\Cref{sec:mathematical_reasoning_results}).
The output is natural language, making it inherently interpretable.

CODI~\citep{shen2025codi} performs trajectory-level global distillation, aligning the student's hidden state to the teacher's at the answer position. Whether this preserves the full reasoning content depends on how much the teacher's hidden state actually encodes about reasoning span. Since there is no structural guarantee, we mark it as not lossless. Supervision acts only at the trajectory endpoint, and recent analyses show that this leads to pervasive shortcut behavior. Models trained with global losses can achieve high accuracy without meaningfully relying on intermediate latent tokens~\citep{zhang2025latent, cui2026latent}.
Training and inference are both efficient, but without per-token supervision, individual latent tokens are hard to interpret~(\Cref{app:interp_analysis}).

SIM-CoT~\citep{wei2026simcot} and KaVa~\citep{kuzina2026kava} represent two flavors of local supervision.
SIM-CoT attaches an auxiliary autoregressive decoder that reconstructs the full aligned reasoning span from each latent token, providing a lossless training signal.
KaVa instead distills compressed key-value cache states from the teacher via an importance-based eviction mechanism (R-KV) that selectively discards KV pairs, making its supervision lossy.
Both methods mitigate shortcut behavior through step-level local supervision~\citep{cui2026latent}.
KaVa's auxiliary cost is negligible, whereas SIM-CoT's decoder adds $16$--$32$\% training overhead (\Cref{app:training_cost}).
Both can produce interpretable latent tokens: SIM-CoT through its decoder output, and KaVa through vocabulary projection of its distilled representations.

\ours{} is the only method in this comparison that satisfies all five properties. The multiplexed targets are provably lossless under suitable positional weightings (\Cref{thm:slot,pro:weightings}), and the non-collapsing guarantee (\Cref{thm:anticollapse}) prevents shortcut behavior. Training adds only a KL divergence over vocabulary distributions ($<$0.01\% of the base cost (\Cref{app:training_cost})), and inference uses the same compact latent token budget as other continuous methods. Each latent token can be read out through the pretrained unembedding layer, giving a vocabulary distribution that reflects the aligned reasoning content (\Cref{app:interp_analysis}).

\FloatBarrier
\section{Supplementary results}
\label{app:additional}

\subsection{Contribution of local distillation}
\label{app:local_distillation}

\begin{table}[t!]
\centering
\caption{Test accuracies (\%) with local distillation only ($\gamma{=}0$).}
\label{tab:local_distillation}
\footnotesize
\setlength{\tabcolsep}{3.8pt}
\renewcommand{\arraystretch}{1.10}
\begin{tabular}{lcccc|cccc}
\toprule
\multirow{2}{*}{\textbf{Method}}
& \multicolumn{4}{c|}{\textbf{GSM8K-AUG}}
& \multicolumn{4}{c}{\textbf{GSM8K-AUG-NL}} \\
\cmidrule(lr){2-5}\cmidrule(lr){6-9}
& \textbf{ID} & \textbf{SVAMP} & \textbf{GSM-Hard} & \textbf{MultiArith}
& \textbf{ID} & \textbf{SVAMP} & \textbf{GSM-Hard} & \textbf{MultiArith} \\
\midrule
& \multicolumn{8}{c}{\textbf{GPT-2}} \\
\midrule
SIM-CoT ($\gamma{=}0$) & 29.5 & 26.5 & 6.8 & 48.9 & 21.4 & 23.4 & 4.9 & 34.4 \\
\rowcolor{OursBg}
\ours{} ($\gamma{=}0$) & \textbf{38.6} & \textbf{34.4} & \textbf{9.2} & \textbf{75.3} & \textbf{31.9} & \textbf{28.1} & \textbf{7.0} & \textbf{48.3} \\
\midrule
& \multicolumn{8}{c}{\textbf{LLaMA 3.2 1B-Instruct}} \\
\midrule
SIM-CoT ($\gamma{=}0$) & 31.6 & 44.0 & 7.5 & 69.5 & 30.1 & 44.0 & 6.7 & 62.2 \\
\rowcolor{OursBg}
\ours{} ($\gamma{=}0$) & \textbf{48.7} & \textbf{51.2} & \textbf{10.6} & \textbf{98.9} & \textbf{38.9} & \textbf{45.4} & \textbf{9.6} & \textbf{77.0} \\
\midrule
& \multicolumn{4}{c|}{\textbf{LLaMA 3.2 3B} (GSM8K-AUG)} & \multicolumn{4}{c}{\textbf{LLaMA 3.1 8B} (GSM8K-AUG)} \\
\midrule
SIM-CoT ($\gamma{=}0$) & 49.1 & 64.6 & 12.4 & \textbf{100.0} & 45.6 & 69.7 & 11.9 & 95.0 \\
\rowcolor{OursBg}
\ours{} ($\gamma{=}0$) & \textbf{51.5} & \textbf{65.4} & \textbf{12.7} & 97.2 & \textbf{50.6} & \textbf{71.3} & \textbf{13.2} & \textbf{95.6} \\
\bottomrule
\end{tabular}
\end{table}

To isolate the contribution of the local supervision, we remove the global trajectory-level distillation loss by setting $\gamma{=}0$ in both \ours{} and SIM-CoT. \Cref{tab:local_distillation} reports the results across all model-dataset combinations.
\ours{} ($\gamma{=}0$) outperforms SIM-CoT ($\gamma{=}0$) in 23 of 24 settings.
These results imply that the multiplexed target is the source of performance gain of \ours{}, as without any trajectory-level supervision signal, multiplexed local supervision outperforms SIM-CoT's autoregressive decoder-based local supervision. Notably, \ours{} achieves this with a simpler architecture, since no auxiliary decoder is needed.

\FloatBarrier

\subsection{Interpretability analysis}
\label{app:interp_analysis}
\begin{figure}[t!]
\centering
\begin{subfigure}[t]{\linewidth}
\centering
\vspace{0pt}%
\includegraphics[width=\linewidth]{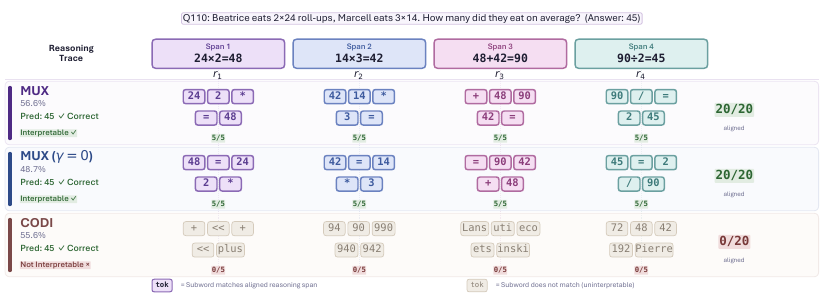}
\caption{Mathematical reasoning (GSM8K-AUG)}
\label{fig:interp_case_math}
\end{subfigure}\hfill
\begin{subfigure}[t]{\linewidth}
\centering
\vspace{0pt}%
\includegraphics[width=\linewidth]{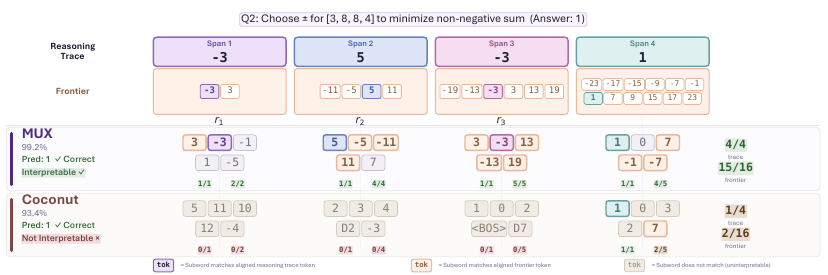}
\caption{Parallel search (MNNS)}
\label{fig:interp_case_mnns}
\end{subfigure}
\caption{Top-5 LM-head decoded subwords per latent token.}
\label{fig:interp_case}
\end{figure}

Prior works interpret latent reasoning by projecting latent tokens back into vocabulary space~\citep{shen2025codi, wei2026simcot, kuzina2026kava}, arguing that interpretability itself signals quality of latent reasoning~\citep{dilgren2026latent}. \Cref{fig:interp_case} shows representative examples of this analysis applied to our method.
\ours{} produces interpretable latent tokens in both mathematical reasoning and parallel search settings. For math reasoning, the top decoded subwords correspond to the operands, operators, and intermediate results of each step. For parallel search, the decoded tokens recover the BFS frontier at each depth, and confirm that a single latent token maintains multiple hypotheses in superposition.

In contrast, Coconut and CODI predict the correct answers, but their decoded tokens are uninformative and do not align with the reasoning spans.
The latent tokens from \ours{} are not only useful for prediction, but also easier to read out.

We complement this with quantitative metrics measured across all test examples. Following the vocabulary-projection probing approach used in prior work~\citep{hao2025training, dilgren2026latent}, for each latent token~$\mathbf{x}_i$ we project it through pretrained unembedding layer of the language model to obtain a distribution over the vocabulary and extract the top-$N$ decoded subwords. We compare these against the reference discrete reasoning span~$\mathbf{r}_i$ aligned to slot~$i$ and report three metrics:
\begin{itemize}[leftmargin=*]
\item \emph{Recall@$N$}: the fraction of tokens in~$\mathbf{r}_i$ that appear among the top-$N$ decoded subwords.
\item \emph{Step Alignment}: the fraction of slots for which the top-$N$ decoded set has its highest token overlap with the correct (diagonally aligned) reasoning step.
\item \emph{MRR} (Mean Reciprocal Rank): the average of $1/\text{rank}$ over all tokens in~$\mathbf{r}_i$, where rank is determined by the decoded vocabulary distribution.
\end{itemize}
All metrics are micro-averaged over slots and examples. We use $N{=}5$ throughout.

\begin{figure}[t!]
\centering
\begin{subfigure}[t]{0.4\textwidth}
\centering
\includegraphics[width=\linewidth]{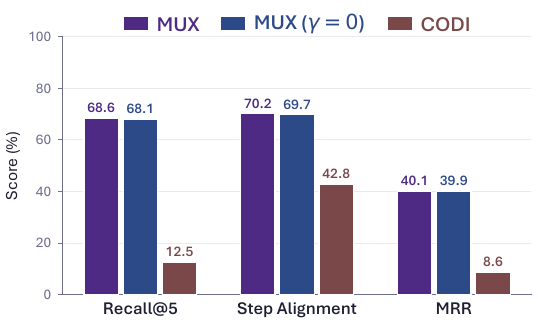}
\caption{Mathematical reasoning (GSM8K-AUG)}
\label{fig:interp_metrics}
\end{subfigure}\hfill
\begin{subfigure}[t]{0.56\textwidth}
\centering
\includegraphics[width=\linewidth]{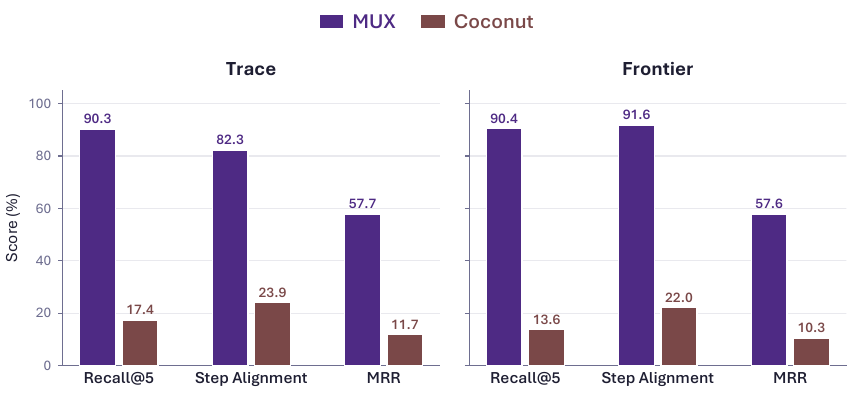}
\caption{Parallel search (MNNS)}
\label{fig:interp_search}
\end{subfigure}
\caption{Quantitative interpretability results.}
\label{fig:interp_metrics_combined}
\end{figure}

\paragraph{Mathematical reasoning.} \Cref{fig:interp_metrics} reports results on LLaMA~3.2 1B-Instruct trained on GSM8K-AUG, evaluated over all test examples. \ours{} recovers $68.6\%$ of reference tokens (Recall@5) and achieves $70.2\%$ Step Alignment, confirming that latent tokens encode both the content and position of their aligned spans. Removing the global loss (\ours{} ($\gamma{=}0$) ) yields nearly identical scores ($68.1\%$, $69.7\%$), pointing to local multiplexed supervision as the driver of interpretability. CODI, trained with only global distillation, reaches $12.5\%$ Recall@5 and $8.6\%$ MRR, consistent with its latent tokens not preserving readable reasoning content.

\paragraph{Parallel search.} To evaluate whether latent tokens encode search structure, we apply the same metrics to the MNNS task. We define two reference targets for each latent token at depth~$k$: the \emph{trace}, which is the single partial sum along the optimal reasoning path at step~$k$, and the \emph{frontier}, which is the complete set of all partial sums reachable at depth~$k$ over every possible sign assignment. Trace metrics test whether the model recovers the particular solution path; frontier metrics test whether the latent state encodes the full distribution of reachable states at each depth. \Cref{fig:interp_search} shows the results. \ours{} achieves $90.3\%$ trace Recall@5 and $82.3\%$ trace Step Alignment, compared to Coconut's $17.4\%$ and $23.9\%$. The pattern is equally strong for frontier metrics ($90.4\%$ vs.\ $13.6\%$ Recall@5; $91.6\%$ vs.\ $22.0\%$ Step Alignment). These numbers show that \ours{} latent tokens encode both the solution path and the reachable states at each depth, consistent with the analysis in \Cref{sec:parallel-search}.

\subsection{Attention analysis}
\label{app:attention_analysis}

\begin{figure}[t!]
    \centering
    \includegraphics[width=\linewidth]{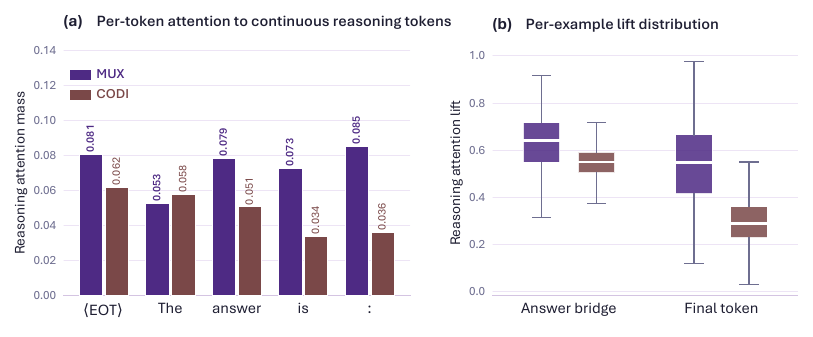}
    \centering
    \caption{Attention analysis on GSM8K-AUG.}
    \label{fig:attention_analysis}
\end{figure}

We add an attention-based diagnostic to test whether latent reasoning meaningfully contributes to final answer. On LLaMA~3.2 1B-Instruct trained on GSM8K-AUG, we measure how much the model attends to its latent tokens when producing the answer. We extract last-layer attention on all test examples at the answer-interface tokens $\{\texttt{<EOT>}, \texttt{The}, \texttt{answer}, \texttt{is}, \texttt{:}\}$ and compute two quantities. Let $\alpha_{\mathrm{lat}}$ be the total attention weight on all $K$ latent tokens, and let $N_{\mathrm{pre}}$ be the number of total preceding tokens. We define \emph{reasoning attention mass}~$= \alpha_{\mathrm{lat}}$, and \emph{reasoning attention lift}~$= \alpha_{\mathrm{lat}} \,/\, (K / N_{\mathrm{pre}})$. A lift of $1.0$ means the model distributes attention uniformly and values above $1.0$ mean latent tokens receive more attention than their share within the context.

\Cref{fig:attention_analysis} reports the results. Panel~(a) shows that \ours{} assigns higher reasoning attention mass than CODI across almost every answer-prediction step. Panel~(b) compares the per-example distribution of reasoning attention lift at two scopes: the full answer bridge (all five interface tokens) and the final prediction token~(\texttt{:}). \ours{} achieves higher lift in both cases ($0.633$ vs.\ $0.542$ at the answer bridge; $0.544$ vs.\ $0.295$ at the final token). In addition, on $85.1\%$ of examples \ours{} routes more attention through latent reasoning tokens at the answer bridge, rising to $91.8\%$ at the final prediction token. Thus, relative to CODI, \ours{} more effectively utilizes its learned latent reasoning when generating the answer. We leave a theoretical explanation in \Cref{app:attention_routing}.
\Cref{fig:attention_qualitative} shows a representative last-layer attention map. \ours{} forms a clear autoregressive chain among its latent tokens before the answer bridge; CODI's attention is diffuse and largely bypasses the latent tokens.

\subsection{Training cost analysis}
\label{app:training_cost}

\begin{table}[t!]
\centering
\caption{Training cost relative to CODI (LLaMA-1B)}
\label{tab:training_cost}
\small
\setlength{\tabcolsep}{5pt}
\begin{tabular}{l cc}
\toprule
& \multicolumn{2}{c}{\textbf{Relative training cost} (vs.\ CODI)} \\
\cmidrule(lr){2-3}
\textbf{Method} & \textbf{GSM8K-AUG} & \textbf{GSM8K-AUG-NL} \\
\midrule
SFT-CoT & 0.56$\times$ & 0.64$\times$ \\
Coconut & 0.44$\times$ & 0.36$\times$ \\
CODI & 1.00$\times$ & 1.00$\times$ \\
SIM-CoT & 1.16$\times$ & 1.32$\times$ \\
KaVa & $\approx$1.00$\times$ & $\approx$1.00$\times$ \\
\rowcolor{OursBg} \ours{} & $\approx$1.00$\times$ & $\approx$1.00$\times$ \\
\bottomrule
\end{tabular}
\end{table}

We compare the per-step training FLOPs of each method using the standard approximation~\citep{kaplan2020scaling}: the forward pass costs $\approx 2PL$ and the backward pass $\approx 4PL$, giving a total of $\approx 6PL$ FLOPs per training step, where $P$ is the number of model parameters and $L$ is the sequence length.
All self-distillation methods (CODI, SIM-CoT, KaVa, and \ours{}) process both a teacher sequence of length $L_t = L_q + L_c + L_a$ (question, full chain-of-thought, answer) and a student sequence of length $L_s = L_q + K + L_a$ (question, $K$ continuous tokens, answer). Since the teacher cross-entropy loss is included in the total training loss and gradients flow through both paths, each incurs the full $6PL$ training cost. SFT-CoT trains only on the teacher sequence, and Coconut trains only on the student sequence.

The methods differ only in their auxiliary losses. SIM-CoT~\citep{wei2026simcot} trains a full auxiliary decoder with $P_{\text{dec}} = P$ parameters on the chain-of-thought tokens, adding $6P_{\text{dec}} \cdot L_c$ FLOPs per step. KaVa~\citep{kuzina2026kava} adds a KV-cache matching loss (Eq.~7 in their paper) with cost $\mathcal{O}(MHLd)$ where $M$ is the number of retained KV pairs, $H$ the number of KV heads, $L$ the number of layers, and $d$ the head dimension. \ours{} adds a multiplexed KL divergence over $K$ vocabulary distributions, costing $\mathcal{O}(K|\mathcal{V}|)$. Both KaVa's and \ours{}'s auxiliary costs are negligible ($<$0.01\% of the base cost). However, KaVa's KV-cache distillation requires an importance-based eviction mechanism (R-KV) that scores and selectively discards teacher KV pairs before matching, introducing additional architectural complexity and a lossy compression step that is absent in \ours{}.

\Cref{tab:training_cost} reports the total relative training cost for LLaMA-1B on both GSM8K-AUG ($L_q{=}55$, $L_c{=}25$, $L_a{=}8$, $K{=}6$) and GSM8K-AUG-NL ($L_q{=}55$, $L_c{=}62$, $L_a{=}8$, $K{=}6$). CODI, KaVa, and \ours{} have effectively identical training cost on both datasets. SIM-CoT is 16\% more expensive on AUG and 32\% on AUG-NL, as its decoder overhead scales with chain length.

\clearpage
\begin{figure}[p]
    \thispagestyle{empty}
    \centering
    \includegraphics[width=\linewidth]{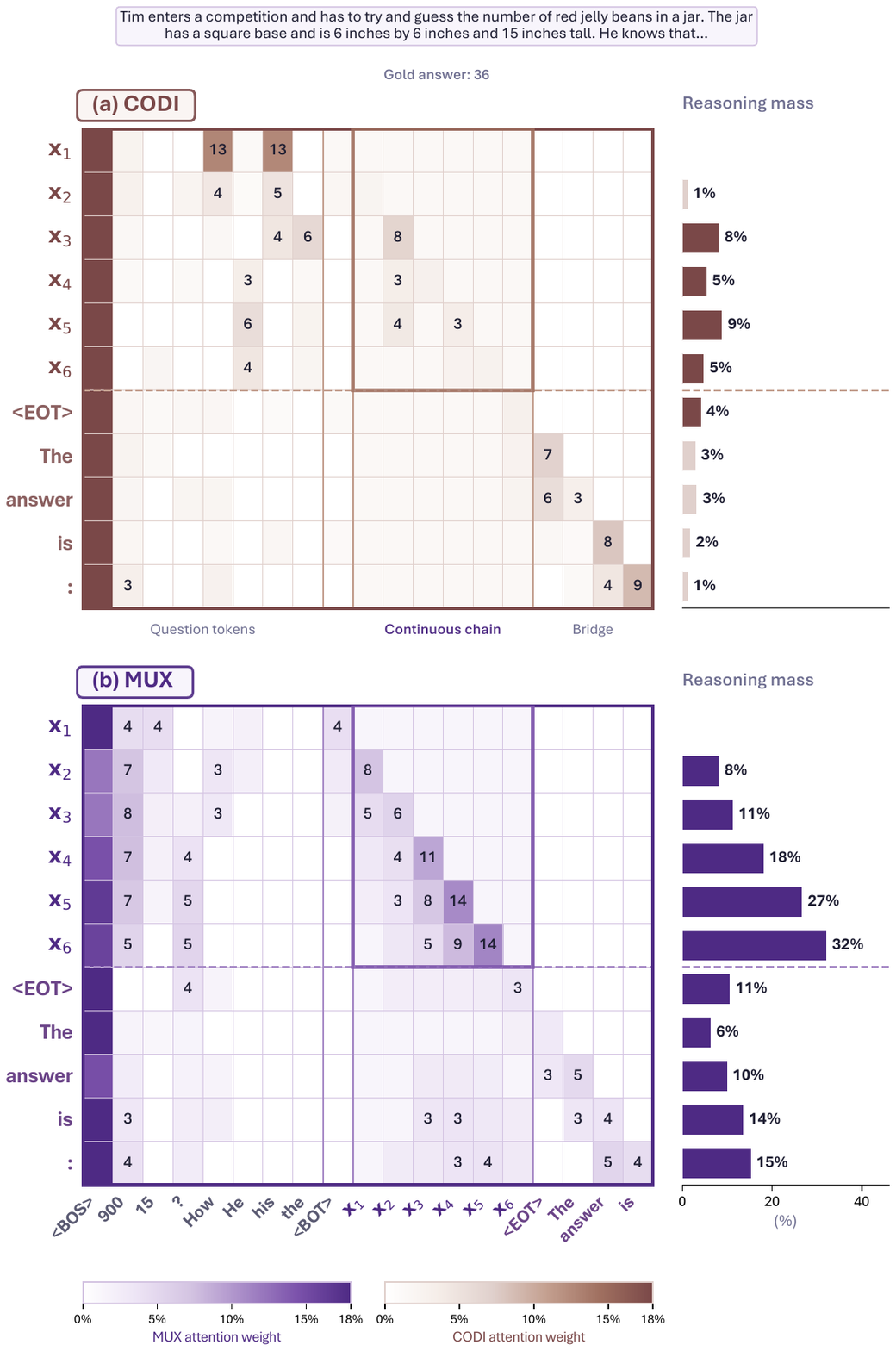}
    \caption{Attention routing through continuous reasoning tokens.}
    \label{fig:attention_qualitative}
\end{figure}

\newpage
\FloatBarrier
\section{Proofs and theoretical details}
\label{app:theory}

\subsection{Proofs of the main results}
\label{app:proofs}

\slot*

\begin{proof}
\label{app:proof_slot}
We prove both directions.

\paragraph{Sufficiency.}
Assume $\Esep(\boldsymbol{\alpha})>0$.
We will show that $\mathsf{mux}$ is injective.

Take two token sequences
\[
(r^1,\dots,r^S), \qquad (r'^1,\dots,r'^S)
\]
such that
\[
\mathsf{mux}(r^1,\dots,r^S)
=
\mathsf{mux}(r'^1,\dots,r'^S).
\]
This means that the two sequences induce exactly the same target distribution over the vocabulary.

For each vocabulary token $v\in \Vocab$, define the set of positions at which $v$ appears:
\[
A_v=\{j\in\{1,\dots,S\}: r^j=v\},
\qquad
B_v=\{j\in\{1,\dots,S\}: r'^j=v\}.
\]
Because the two target distributions are equal, for every $v\in \Vocab$ we have
\[
\sum_{j\in A_v}\alpha_j
=
\sum_{j\in B_v}\alpha_j.
\]
Suppose, for contradiction, that $A_v\neq B_v$ for some $v$.
Then the coefficient vector defined by
\[
c_j=
\begin{cases}
1, & j\in A_v\setminus B_v,\\
-1, & j\in B_v\setminus A_v,\\
0, & \text{otherwise}
\end{cases}
\]
is nonzero and satisfies
\[
\sum_{j=1}^{S}c_j \alpha_j
=
\sum_{j\in A_v}\alpha_j-\sum_{j\in B_v}\alpha_j
=
0.
\]
This contradicts $\Esep(\boldsymbol{\alpha})>0$.
Therefore $A_v=B_v$ for every token $v$.

Now fix any position $j\in\{1,\dots,S\}$.
There is exactly one vocabulary token $v$ such that $j\in A_v$, namely $v=r^j$.
Since $A_v=B_v$, we also have $j\in B_v$, so $r'^j=v=r^j$.
As this holds for every position $j$, the two sequences are identical.
Hence $\mathsf{mux}$ is injective.

\paragraph{Necessity.}
Assume $\Esep(\boldsymbol{\alpha})=0$.
Then, by definition, there exists a nonzero coefficient vector
\[
\mathbf{c}=(c_1,\dots,c_S)\in\{-1,0,1\}^S
\]
such that
\[
\sum_{j=1}^{S} c_j \alpha_j=0.
\]
Define two subsets
\[
A=\{j:c_j=1\},
\qquad
B=\{j:c_j=-1\}.
\]
Since $\mathbf{c}\neq\mathbf{0}$, at least one of $A$ or $B$ is non-empty.
Moreover,
\[
\sum_{j\in A}\alpha_j=\sum_{j\in B}\alpha_j.
\]
Choose two distinct vocabulary tokens $u,v\in \Vocab$.
Construct two sequences by
\[
r^j=
\begin{cases}
u, & j\in A,\\
v, & j\notin A,
\end{cases}
\qquad
r'^j=
\begin{cases}
u, & j\in B,\\
v, & j\notin B.
\end{cases}
\]
Because $A\neq B$, the sequences are different.
However, the probability mass of token $u$ under the first sequence is $\sum_{j\in A}\alpha_j$, while under the second sequence it is $\sum_{j\in B}\alpha_j$; these are equal.
The same is true for token $v$, since both distributions sum to $1$, and all other tokens have probability $0$.
Therefore the two sequences induce exactly the same target distribution.
Hence $\mathsf{mux}$ is not injective.

We have shown that $\mathsf{mux}$ is injective if and only if $\Esep(\boldsymbol{\alpha})>0$.
\end{proof}

\chain*
\begin{proof}
\label{app:proof_chain}
Fix \(i\in\{1,\ldots,M\}\). By assumption,
\[
\Esep(\boldsymbol\alpha^{(i)})>0.
\]
Therefore, by \Cref{thm:slot}, the multiplexed target \(\muxop(\mathbf r_i)\), together with the span length \(S_i\) and the corresponding masses \(\boldsymbol\alpha^{(i)}\), uniquely determines the full aligned span
\[
\mathbf r_i=(r_i^1,\dots,r_i^{S_i}).
\]
This is true for every span \(i=1,\ldots,M\).

Once all spans \(\mathbf r_i\) have been recovered, the original reasoning trace is obtained by concatenating them in the same order. Thus the ordered tuple
\[
\Bigl((S_i,\boldsymbol\alpha^{(i)},\muxop(\mathbf r_i))\Bigr)_{i=1}^{M}
\]
determines the full reasoning trace uniquely.
\end{proof}

\proWeightings*
\begin{proof}
For geometric weighting,
\[
\alpha_j=\frac{\rho^{j-1}}{\sum_{l=1}^{S}\rho^{l-1}}.
\]
Let
\[
Z_S=\sum_{l=1}^{S}\rho^{l-1}.
\]
Since $0<\rho<1$, we have $Z_S>0$.

Take any coefficient vector $\mathbf{c}\in\{-1,0,1\}^{S}$.
Then
\[
\sum_{j=1}^{S} c_j \alpha_j
=
\sum_{j=1}^{S} c_j\frac{\rho^{j-1}}{Z_S}
=
\frac{1}{Z_S}\sum_{j=1}^{S} c_j\rho^{j-1}.
\]
Because $Z_S>0$, this quantity is zero if and only if
\[
\sum_{j=1}^{S} c_j\rho^{j-1}=0.
\]
Therefore
\[
\Esep(\boldsymbol{\alpha})>0
\quad\Longleftrightarrow\quad
\sum_{j=1}^{S} c_j\rho^{j-1}\neq 0
\quad
\text{for every }\mathbf{c}\in\{-1,0,1\}^{S}\setminus\{\mathbf{0}\}.
\]
\Cref{thm:slot} now gives the stated equivalence.

To prove that if $\rho\in(0,1)$ is rational, then geometric weighting is injective for every finite span length $S$, suppose $\rho=p/q\in(0,1)$ is rational in lowest terms and that geometric weighting were not injective.
We proved that there would exist a nonzero polynomial
\[
P(x)=\sum_{j=1}^{S} c_j x^{j-1},
\qquad c_j\in\{-1,0,1\},
\]
such that $P(\rho)=0$.
If necessary, divide out the largest power of $x$ so that the constant term is nonzero.
The resulting polynomial still has integer coefficients, is nonzero, has constant term $\pm 1$, and has leading coefficient $\pm 1$.
By the rational root theorem, any rational root must be an integer divisor of the constant term divided by an integer divisor of the leading coefficient, hence must belong to $\{\pm 1\}$.
This contradicts $\rho\in(0,1)$.
Therefore no such polynomial exists, and the weighting is injective.

To prove part (ii), let us introduce the following lemma on exponential polynomials first.

\begin{lemma}
\label{lem:exp_poly}
Let $\lambda_1,\dots,\lambda_n\in\mathbb{R}$ be pairwise distinct, and let
\[
f(x)=\sum_{m=1}^{n} b_m e^{\lambda_m x}
\]
with real coefficients $b_m$, not all zero.
Then $f$ has at most $n-1$ real zeros.
\end{lemma}

\begin{proof}
We use induction on $n$.

If $n=1$, then
\[
f(x)=b_1 e^{\lambda_1 x}
\]
with $b_1\neq 0$, so $f(x)\neq 0$ for all $x$.
Thus the claim holds.

Assume the statement holds for $n-1$, and consider
\[
f(x)=\sum_{m=1}^{n} b_m e^{\lambda_m x}
\quad\text{with}\quad
\lambda_1<\lambda_2<\cdots<\lambda_n.
\]
Define
\[
g(x)=e^{-\lambda_1 x}f(x)=b_1+\sum_{m=2}^{n} b_m e^{(\lambda_m-\lambda_1)x}.
\]
The functions $f$ and $g$ have the same zeros because $e^{-\lambda_1 x}$ is never zero.

Suppose $g$ has $N$ distinct real zeros.
By Rolle's theorem, $g'$ has at least $N-1$ distinct real zeros.
But
\[
g'(x)=\sum_{m=2}^{n} b_m(\lambda_m-\lambda_1)e^{(\lambda_m-\lambda_1)x}
\]
is again an exponential polynomial, now with $n-1$ pairwise distinct exponents.
By the induction hypothesis, $g'$ has at most $n-2$ real zeros.
Therefore $N-1\le n-2$, which implies $N\le n-1$.

Hence $g$, and therefore also $f$, has at most $n-1$ real zeros.
\end{proof}

Now suppose the scores $s_1,\dots,s_S$ are pairwise distinct, and define
\[
\alpha_j(\lambda)=\frac{e^{\lambda s_j}}{\sum_{l=1}^{S}e^{\lambda s_l}}.
\]
By \Cref{thm:slot}, injectivity fails if and only if there exists a nonzero coefficient vector
\[
\mathbf{c}=(c_1,\dots,c_S)\in\{-1,0,1\}^{S}
\]
such that
\[
\sum_{j=1}^{S} c_j \alpha_j(\lambda)=0.
\]
Since the denominator $\sum_l e^{\lambda s_l}$ is strictly positive, this is equivalent to
\[
\sum_{j=1}^{S} c_j e^{\lambda s_j}=0.
\]
For fixed nonzero $\mathbf{c}$, the function
\[
f_{\mathbf{c}}(\lambda)=\sum_{j=1}^{S} c_j e^{\lambda s_j}
\]
is a nonzero exponential polynomial with pairwise distinct exponents $s_j$.
By \Cref{lem:exp_poly}, $f_{\mathbf{c}}$ has only finitely many real zeros.

There are only finitely many nonzero coefficient vectors in $\{-1,0,1\}^{S}$.
Therefore the union of the zero sets of all such functions $f_{\mathbf{c}}$ is finite.
Call this union $D$.
If $\lambda\notin D$, then no nontrivial signed sum vanishes, so $\Esep(\boldsymbol{\alpha})>0$.
By \Cref{thm:slot}, the encoding is injective.
\end{proof}

\corSinRot*

\begin{proof}
    
Part (i) is immediate because the function
\[
u\mapsto \sin\left(\frac{\pi}{2}u\right)
\]
is strictly increasing on $[0,1]$, so the sinusoidal scores are pairwise distinct.
For part (ii), assume
\[
0<\theta_p(S-1)<\pi
\qquad\text{for every }p=1,\dots,P.
\]
Fix $p\in\{1,\dots,P\}$.
For each $j=1,\dots,S-1$,
\[
0\le (j-1)\theta_p < j\theta_p < \pi.
\]
The cosine function is strictly decreasing on the interval $[0,\pi]$.
Hence
\[
\cos\bigl(\theta_p(j-1)\bigr)>\cos(\theta_p j)
\qquad\text{for }j=1,\dots,S-1.
\]
Averaging these inequalities over $p\in\{1,\dots,P\}$ gives
\[
\frac{1}{P}\sum_{p=1}^{P}\cos\bigl(\theta_p(j-1)\bigr)
>
\frac{1}{P}\sum_{p=1}^{P}\cos(\theta_p j),
\]
that is,
\[
s_j>s_{j+1}
\qquad\text{for }j=1,\dots,S-1.
\]
Thus the rotary scalar scores are strictly decreasing and therefore pairwise distinct. The injectivity claim then follows immediately from \Cref{pro:weightings}.
\end{proof}
\anticollapse*

\begin{proof}
Let
\[
n:=|\mathcal{K}|,\qquad
\widetilde W := W/\tau,\qquad
m_i:=\mathsf{mux}({\bf r}_i),\qquad
q_i:=f({\bf x}_i)=\mathrm{softmax}(\widetilde W{\bf x}_i),
\qquad
e_i:=\|m_i-q_i\|_1.
\]

By Pinsker's inequality,
\[
e_i \le \sqrt{2\,D_{\mathrm{KL}}(m_i\,\|\,q_i)}
\qquad\text{for every } i\in\mathcal{K}.
\]
Hence, by Jensen's inequality and the assumption $\mathcal{L}_{\mathrm{local}}\le \delta$,
\[
\frac{1}{n}\sum_{i\in\mathcal{K}} e_i
\le
\frac{1}{n}\sum_{i\in\mathcal{K}} \sqrt{2\,D_{\mathrm{KL}}(m_i\,\|\,q_i)}
\le
\sqrt{2\cdot \frac{1}{n}\sum_{i\in\mathcal{K}} D_{\mathrm{KL}}(m_i\,\|\,q_i)}
\le
\sqrt{2\delta}.
\]

For any distinct $i,j\in\mathcal{K}$, \Cref{def:target_diversity} and the triangle inequality give
\[
\|q_i-q_j\|_1
\ge
\|m_i-m_j\|_1 - \|m_i-q_i\|_1 - \|m_j-q_j\|_1
\ge
\mathcal{D} - e_i - e_j.
\]
Averaging over all ordered pairs $i\neq j$ yields
\[
\frac{1}{n(n-1)}\sum_{i\neq j}\|q_i-q_j\|_1
\ge
\mathcal{D}
-
\frac{1}{n(n-1)}\sum_{i\neq j}(e_i+e_j).
\]
Since
\[
\frac{1}{n(n-1)}\sum_{i\neq j}(e_i+e_j)
=
\frac{2}{n}\sum_{i\in\mathcal{K}} e_i,
\]
we obtain
\[
\frac{1}{n(n-1)}\sum_{i\neq j}\|q_i-q_j\|_1
\ge
\mathcal{D} - \frac{2}{n}\sum_{i\in\mathcal{K}} e_i
\ge
\mathcal{D} - 2\sqrt{2\delta}.
\]
Because $\delta<\mathcal{D}^2/8$, the right-hand side is strictly positive.

Let $\sigma$ denote the softmax map. For $\mathbf z\in\mathbb R^{|\Vocab|}$, write
\[
J_{\sigma}(\mathbf z)
=
\operatorname{Diag}(\sigma(\mathbf z))-\sigma(\mathbf z)\sigma(\mathbf z)^{\!\top}
\]
for its Jacobian matrix, and define
\[
C_{|\Vocab|}
:=
\sup_{\mathbf z\in\mathbb{R}^{|\Vocab|}}
\|J_{\sigma}(\mathbf z)\|_{2\to 1},
\qquad
\|A\|_{2\to 1}:=\sup_{\|\mathbf v\|_2=1}\|A\mathbf v\|_1.
\]
Then $C_{|\Vocab|}$ depends only on $|\Vocab|$. By the mean value theorem, for every $i,j$,
\[
\|q_i-q_j\|_1
=
\|\sigma(\widetilde W{\bf x}_i)-\sigma(\widetilde W{\bf x}_j)\|_1
\le
C_{|\Vocab|}\,\|\widetilde W({\bf x}_i-{\bf x}_j)\|
\le
C_{|\Vocab|}\,\|\widetilde W\|_{\mathrm{op}}\,\|{\bf x}_i-{\bf x}_j\|.
\]
Therefore
\[
\frac{1}{n(n-1)}\sum_{i\neq j}\|{\bf x}_i-{\bf x}_j\|
\ge
\frac{\mathcal{D}-2\sqrt{2\delta}}{\|\widetilde W\|_{\mathrm{op}}\,C_{|\Vocab|}}.
\]
Applying Jensen's inequality,
\[
\frac{1}{n(n-1)}\sum_{i\neq j}\|{\bf x}_i-{\bf x}_j\|^2
\ge
\left(
\frac{1}{n(n-1)}\sum_{i\neq j}\|{\bf x}_i-{\bf x}_j\|
\right)^2
\ge
\left(
\frac{\mathcal{D}-2\sqrt{2\delta}}{\|\widetilde W\|_{\mathrm{op}}\,C_{|\Vocab|}}
\right)^2.
\]
Since the diagonal terms vanish,
\[
\frac{1}{n^2}\sum_{i,j\in\mathcal{K}}\|{\bf x}_i-{\bf x}_j\|^2
=
\frac{n-1}{n}\cdot
\frac{1}{n(n-1)}\sum_{i\neq j}\|{\bf x}_i-{\bf x}_j\|^2
\ge
\frac{n-1}{n}
\left(
\frac{\mathcal{D}-2\sqrt{2\delta}}{\|\widetilde W\|_{\mathrm{op}}\,C_{|\Vocab|}}
\right)^2.
\]
This proves the claimed lower bound on the average pairwise squared distance. Hence the continuous tokens cannot exhibit representation collapse at any level $\varepsilon$ below this quantity. Since the average of these nonnegative squared distances is at least this quantity, there exists a distinct pair $i,j\in\mathcal{K}$ such that
\[
\|{\bf x}_i-{\bf x}_j\|
\ge
\sqrt{\frac{n-1}{n}}\,
\frac{\mathcal{D}-2\sqrt{2\delta}}{\|\widetilde W\|_{\mathrm{op}}\,C_{|\Vocab|}}.
\]
\end{proof}

\parallel*

\begin{proof}
We construct a recurrence over continuous tokens and verify that it exactly implements breadth-first search.

For a set \(B\subseteq \mathcal{N}\), let \(1_B\in\{0,1\}^n\) denote its indicator vector, where \(n=|\mathcal{N}|\). At step $k$, let the continuous token be the pair
\[
(f_k,u_k)\in\{0,1\}^{2n},
\]
where
\[
f_k = 1_{F_k},
\qquad
u_k = 1_{U_k}.
\]
Thus the token stores the current frontier and the set of visited nodes.

Initialize
\[
f_0 = 1_{\{s\}},
\qquad
u_0 = 1_{\{s\}}.
\]

Let $A \in \{0,1\}^{n\times n}$ be the adjacency matrix of the graph, with
\[
A_{uv} = 1 \iff (u,v)\in E.
\]
Given the token $(f_k,u_k)$, define the next token by
\[
g_{k+1} = 1[A^\top f_k > 0],
\]
\[
f_{k+1} = g_{k+1}\odot (1-u_k),
\]
\[
u_{k+1} = u_k + f_{k+1}.
\]

We claim that for every $k\le H$,
\[
f_k = 1_{F_k},
\qquad
u_k = 1_{U_k}.
\]

The claim is immediate at $k=0$. Assume it holds at step $k$. For any node $v\in \mathcal{N}$,
\[
(g_{k+1})_v = 1
\iff
(A^\top f_k)_v > 0
\iff
\exists\, u\in F_k \text{ such that } (u,v)\in E
\iff
v\in N^+(F_k).
\]
Therefore
\[
g_{k+1} = 1_{N^+(F_k)}.
\]
Hence
\[
f_{k+1}
=
1_{N^+(F_k)} \odot (1-1_{U_k})
=
1_{N^+(F_k)\setminus U_k}
=
1_{F_{k+1}}.
\]
Also, by the BFS update, $F_{k+1}\cap U_k=\varnothing$, so
\[
u_{k+1}
=
u_k + f_{k+1}
=
1_{U_k} + 1_{F_{k+1}}
=
1_{U_k\cup F_{k+1}}
=
1_{U_{k+1}}.
\]
This proves the claim by induction.

It follows that the recurrence exactly tracks the breadth-first frontier and visited set at every step. In particular, the final answer is exact:
\[
y = 1[t\in U_H].
\]
Indeed, since $u_H = 1_{U_H}$ is part of the final token, the answer head can read the coordinate corresponding to $t$ and output the correct answer.

It remains to recover the frontier distribution. If $F_k=\varnothing$, one may use a designated null distribution. Assume now that $F_k\neq\varnothing$. Since
\[
f_k = 1_{F_k},
\]
the frontier is explicitly encoded in the token, so define
\[
p_k(v) = \frac{(f_k)_v}{\|f_k\|_1}.
\]
Then
\[
p_k(v)=
\begin{cases}
1/|F_k|, & v\in F_k,\\
0, & v\notin F_k.
\end{cases}
\]
By the setup of \Cref{sec:parallel_search}, this is exactly $\mathrm{mux}(r_k)$.

Finally, if one insists on a standard softmax readout with finite logits, exact zeros outside $F_k$ are impossible, but arbitrarily good approximation is still possible. For $B>0$, define
\[
\ell_k(v)=B\bigl((f_k)_v-1\bigr).
\]
Then
\[
\ell_k(v)=
\begin{cases}
0, & v\in F_k,\\
-B, & v\notin F_k.
\end{cases}
\]
Let $m=|F_k|$. The corresponding softmax distribution is
\[
p_k^{(B)}(v)
=
\frac{e^{\ell_k(v)}}{\sum_{u\in \mathcal{N}} e^{\ell_k(u)}}.
\]
Since
\[
\sum_{u\in \mathcal{N}} e^{\ell_k(u)} = m + (|\mathcal{N}|-m)e^{-B},
\]
we obtain
\[
p_k^{(B)}(v)=
\begin{cases}
\dfrac{1}{m + (|\mathcal{N}|-m)e^{-B}}, & v\in F_k,\\[1.25ex]
\dfrac{e^{-B}}{m + (|\mathcal{N}|-m)e^{-B}}, & v\notin F_k.
\end{cases}
\]
Therefore
\[
p_k^{(B)} \to \mathrm{mux}(r_k)
\qquad\text{as } B\to\infty.
\]
So the frontier distribution is recoverable from the continuous token exactly, and realizable by a standard softmax readout up to arbitrarily small error.
\end{proof}

\subsection{Multiplexing under finite precision}
\label{app:finite_precision}

\Cref{thm:slot} characterizes lossless multiplexing in exact arithmetic: for a fixed span length \(S\), injectivity of \(\muxop:\Vocab^S\to\Delta^{|\Vocab|-1}\) is equivalent to \(\Esep(\boldsymbol\alpha)>0\).
We now make the finite-precision version of this statement explicit. The argument has two steps. First, the same margin \(\Esep(\boldsymbol\alpha)\) is the minimum coordinatewise separation between distinct exact targets. Second, under the standard unit-roundoff model, target construction introduces an \(O(Su)\) perturbation for span length \(S\) and unit roundoff \(u\). We state everything for one fixed span length \(S\); for full traces with varying span lengths, the argument applies spanwise exactly as in \Cref{cor:chain}. We use the \(\ell_\infty\) norm because each coordinate of \(\muxop({\bf r})\) is a subset sum of the masses, and the separation margin in \Cref{def:E} is coordinatewise.

We first identify \(\Esep(\boldsymbol\alpha)\) as the minimum \(\ell_\infty\)-distance between two distinct exact multiplexed targets.

\begin{proposition}[Separation between distinct exact multiplexed targets]
\label{prop:fp_separation}
Assume \(|\Vocab|>1\). Then
\[
\min_{{\bf r}\neq{\bf r}'\in\Vocab^S}
\|\muxop({\bf r})-\muxop({\bf r}')\|_\infty
=
\Esep(\boldsymbol\alpha).
\]
\end{proposition}

\begin{proof}
Take any distinct \({\bf r},{\bf r}'\in\Vocab^S\). For each vocabulary symbol \(v\in\Vocab\),
\[
\bigl(\muxop({\bf r})-\muxop({\bf r}')\bigr)_v
=
\sum_{j=1}^S c_j^{(v)}\alpha_j,
\qquad
c_j^{(v)}=\ind[r^j=v]-\ind[(r')^j=v]\in\{-1,0,1\}.
\]
If \({\bf r}\neq{\bf r}'\), then for at least one \(v\) the coefficient vector
\((c_1^{(v)},\dots,c_S^{(v)})\) is nonzero. By \Cref{def:E},
\[
\left|\bigl(\muxop({\bf r})-\muxop({\bf r}')\bigr)_v\right|
\ge
\Esep(\boldsymbol\alpha),
\]
and hence
\[
\|\muxop({\bf r})-\muxop({\bf r}')\|_\infty\ge \Esep(\boldsymbol\alpha).
\]
Taking the minimum over all distinct pairs yields
\[
\min_{{\bf r}\neq{\bf r}'}
\|\muxop({\bf r})-\muxop({\bf r}')\|_\infty
\ge
\Esep(\boldsymbol\alpha).
\]

For the reverse inequality, choose a nonzero vector
\({\bf c}=(c_1,\dots,c_S)\in\{-1,0,1\}^S\) attaining the minimum in
\eqref{eq:subset_sum}. Since \(|\Vocab|>1\), pick distinct symbols \(u,v\in\Vocab\), and define \({\bf r},{\bf r}'\in\Vocab^S\) by
\[
r^j=
\begin{cases}
u, & c_j=1,\\
v, & c_j\in\{-1,0\},
\end{cases}
\qquad
(r')^j=
\begin{cases}
u, & c_j=-1,\\
v, & c_j\in\{1,0\}.
\end{cases}
\]
Then the only possibly nonzero coordinates of
\(\muxop({\bf r})-\muxop({\bf r}')\) are the \(u\)- and \(v\)-coordinates, equal to
\[
\sum_{j=1}^S c_j\alpha_j
\qquad\text{and}\qquad
-\sum_{j=1}^S c_j\alpha_j,
\]
respectively. Therefore
\[
\|\muxop({\bf r})-\muxop({\bf r}')\|_\infty
=
\left|\sum_{j=1}^S c_j\alpha_j\right|
=
\Esep(\boldsymbol\alpha),
\]
which proves the reverse inequality.
\end{proof}

\Cref{prop:fp_separation} is the exact-arithmetic separation statement. It shows that any perturbation smaller than half of this margin preserves unique demultiplexing.

\begin{corollary}[Stable demultiplexing under bounded perturbation]
\label{thm:fp_stability}
Assume \(|\Vocab|>1\). Let \({\bf r}\in\Vocab^S\), and let \(y\in\mathbb{R}^{|\Vocab|}\) satisfy
\[
\|y-\muxop({\bf r})\|_\infty
<
\frac{\Esep(\boldsymbol\alpha)}{2}.
\]
Then \({\bf r}\) is the unique minimum-\(\ell_\infty\) demultiplexing of \(y\), i.e.
\[
\arg\min_{{\bf s}\in\Vocab^S}\|y-\muxop({\bf s})\|_\infty
=
\{{\bf r}\}.
\]
\end{corollary}

\begin{proof}
Take any competitor \({\bf s}\neq{\bf r}\). By \Cref{prop:fp_separation},
\[
\|\muxop({\bf r})-\muxop({\bf s})\|_\infty\ge \Esep(\boldsymbol\alpha).
\]
Hence the triangle inequality gives
\[
\|y-\muxop({\bf s})\|_\infty
\ge
\|\muxop({\bf r})-\muxop({\bf s})\|_\infty
-
\|y-\muxop({\bf r})\|_\infty
>
\Esep(\boldsymbol\alpha)-\frac{\Esep(\boldsymbol\alpha)}{2}
=
\frac{\Esep(\boldsymbol\alpha)}{2}.
\]
On the other hand,
\[
\|y-\muxop({\bf r})\|_\infty
<
\frac{\Esep(\boldsymbol\alpha)}{2}.
\]
Therefore
\[
\|y-\muxop({\bf r})\|_\infty
<
\|y-\muxop({\bf s})\|_\infty
\qquad\text{for every }{\bf s}\neq{\bf r},
\]
so \({\bf r}\) is the unique minimizer.
\end{proof}

\Cref{thm:fp_stability} applies to any perturbation \(y\) near an exact multiplexed target, regardless of its source. In this paper we use it for finite-precision target construction. Let \(\widetilde{\mathsf{mux}}({\bf r})\) denote the target materialized by the implementation, and define the worst-case target-construction error
\[
\varepsilon_{\mathrm{fp}}
:=
\sup_{{\bf r}\in\Vocab^S}
\|\widetilde{\mathsf{mux}}({\bf r})-\muxop({\bf r})\|_\infty.
\]
This is a target-side quantity: it can include rounding of the masses, approximate normalization, and summation error.
We now make \(\varepsilon_{\mathrm{fp}}\) explicit under the standard unit-roundoff model.
Let \(u\) denote the unit roundoff, and define
\[
\gamma_n(u):=\frac{nu}{1-nu},
\qquad nu<1.
\]
Assume the exact normalized masses \(\alpha_j\) are fixed first, and that:
\begin{enumerate}[label=(\roman*)]
    \item each \(\alpha_j\) is stored once in the working format as \(\widehat{\alpha}_j\), with
    \[
    |\widehat{\alpha}_j-\alpha_j|\le u\alpha_j;
    \]
    \item each coordinate of \(\widetilde{\mathsf{mux}}({\bf r})\) is formed by naively summing the relevant stored masses \(\widehat{\alpha}_j\) in the same arithmetic.
\end{enumerate}
This isolates the floating-point error after the exact masses are fixed.

\begin{corollary}[Floating-point sufficient condition]
\label{cor:fp_format_aware}
Under the model above,
\[
\varepsilon_{\mathrm{fp}}
\le
\eta_S(u)
:=
u+(1+u)\gamma_{S-1}(u)
=
\frac{Su}{1-(S-1)u}.
\]
Consequently, exact recovery is guaranteed whenever
\[
\eta_S(u)<\frac{\Esep(\boldsymbol\alpha)}{2}.
\]
\end{corollary}

\begin{proof}
Fix \({\bf r}\in\Vocab^S\) and a vocabulary symbol \(v\in\Vocab\). Let
\[
I_v:=\{j:r^j=v\},
\qquad
x_v:=\sum_{j\in I_v}\alpha_j,
\qquad
\widehat{x}_v:=\sum_{j\in I_v}\widehat{\alpha}_j.
\]
If \(I_v=\varnothing\), then \(x_v=\widehat{x}_v=0\), so the bound is trivial. Assume \(I_v\neq\varnothing\). Since all terms are nonnegative,
\[
|\widehat{x}_v-x_v|
\le
\sum_{j\in I_v}|\widehat{\alpha}_j-\alpha_j|
\le
u\sum_{j\in I_v}\alpha_j
=
ux_v.
\]
Let \(\widetilde{x}_v\) be the value obtained by naively summing the stored masses \(\widehat{\alpha}_j\). Standard floating-point summation bounds give
\[
\widetilde{x}_v=\widehat{x}_v(1+\theta_{|I_v|-1}),
\qquad
|\theta_{|I_v|-1}|\le \gamma_{|I_v|-1}(u)\le \gamma_{S-1}(u).
\]
Therefore
\[
|\widetilde{x}_v-\widehat{x}_v|
\le
\gamma_{S-1}(u)\,\widehat{x}_v
\le
(1+u)\gamma_{S-1}(u)\,x_v,
\]
where we used \(\widehat{x}_v\le (1+u)x_v\). Combining the two bounds yields
\[
|\widetilde{x}_v-x_v|
\le
\bigl(u+(1+u)\gamma_{S-1}(u)\bigr)x_v
\le
u+(1+u)\gamma_{S-1}(u).
\]
Taking the maximum over \(v\) proves
\[
\varepsilon_{\mathrm{fp}}
\le
u+(1+u)\gamma_{S-1}(u)
=
\frac{Su}{1-(S-1)u}.
\]
The recovery condition then follows from \Cref{thm:fp_stability}.
\end{proof}

For round-to-nearest arithmetic,
\[
u_{\mathrm{FP32}}=2^{-24}\approx 5.96\times 10^{-8}.
\]
Hence, for \(2\le S\le 32\),
\[
\eta_S(u_{\mathrm{FP32}})\le 1.91\times 10^{-6},
\]

\paragraph{Geometric weights.}
The floating-point bound above is independent of the weighting family. The weighting enters only through \(\Esep(\boldsymbol\alpha)\). For rational geometric weights, this margin admits an exact integer-arithmetic representation.

\begin{corollary}[Rational geometric weights]
\label{cor:geo_explicit_margin}
Suppose \(\rho=p/q\in(0,1)\) is rational in lowest terms and
\[
\alpha_j=\frac{\rho^{j-1}}{\sum_{\ell=0}^{S-1}\rho^\ell},
\qquad j=1,\dots,S.
\]
Then
\[
\Esep(\boldsymbol\alpha)
=
\frac{q-p}{q^S-p^S}\,m_S,
\]
where
\[
m_S
:=
\min_{{\bf c}\in\{-1,0,1\}^S\setminus\{0\}}
\left|
\sum_{j=1}^S c_j p^{j-1}q^{S-j}
\right|.
\]
Moreover \(m_S\ge 1\), and therefore
\[
\Esep(\boldsymbol\alpha)\ge \frac{q-p}{q^S-p^S}.
\]
Consequently, a sufficient condition for exact recovery is
\[
\eta_S(u)<\frac{q-p}{2(q^S-p^S)}.
\]
\end{corollary}

\begin{proof}
Using
\[
\sum_{\ell=0}^{S-1}\Bigl(\frac pq\Bigr)^\ell
=
\frac{q^S-p^S}{q^{S-1}(q-p)},
\]
we can rewrite the normalized masses as
\[
\alpha_j
=
\frac{(q-p)p^{j-1}q^{S-j}}{q^S-p^S}.
\]
Hence, for any nonzero \({\bf c}\in\{-1,0,1\}^S\),
\[
\sum_{j=1}^S c_j\alpha_j
=
\frac{q-p}{q^S-p^S}
\sum_{j=1}^S c_j p^{j-1}q^{S-j}.
\]
Taking absolute values and then the minimum over all nonzero \({\bf c}\) gives the exact formula for \(\Esep(\boldsymbol\alpha)\). The quantity inside the absolute value is an integer. It is nonzero for every nonzero \({\bf c}\), because otherwise \(\sum_{j=1}^S c_j \rho^{j-1}=0\), contradicting \Cref{pro:weightings}(a). Therefore \(m_S\ge 1\), which yields the lower bound. The final condition follows by combining this lower bound with \Cref{cor:fp_format_aware}.
\end{proof}

For our default choice \(\rho=9/10\),
\[
\Esep(\boldsymbol\alpha)
=
\frac{m_S}{10^S-9^S},
\qquad
m_S
=
\min_{{\bf c}\in\{-1,0,1\}^S\setminus\{0\}}
\left|
\sum_{j=1}^S c_j\,9^{j-1}10^{S-j}
\right|.
\]
This quantity can be evaluated exactly offline by integer arithmetic for each span length \(S\) used in practice. For \(S\le 30\), exact evaluation gives
\[
\Esep(\boldsymbol\alpha)\approx 3.98\times 10^{-6}\ \text{at }S=11,
\qquad
\Esep(\boldsymbol\alpha)\approx 1.20\times 10^{-6}\ \text{at }S=12,
\]
By contrast,
\[
\eta_{11}(u_{\mathrm{FP32}})\approx 6.56\times 10^{-7},
\qquad
\eta_{12}(u_{\mathrm{FP32}})\approx 7.15\times 10^{-7},
\]
Therefore, for the default geometric choice \(\rho=0.9\), the conservative certificate from \Cref{cor:fp_format_aware} holds in FP32 up to \(S=11\).

\subsection{Why local distillation preserves answer-side use of latent reasoning}
\label{app:attention_routing}

This section gives an objective-level explanation for the attention pattern observed in \Cref{app:interp_analysis}.
The two auxiliary terms in $\mathcal{L} = \mathcal{L}_{\mathrm{answer}} + \beta\,\mathcal{L}_{\mathrm{local}} + \gamma\,\mathcal{L}_{\mathrm{global}}$ constrain different objects.
The local term $\mathcal{L}_{\mathrm{local}}$ constrains each latent reasoning token $\mathbf{x}_i$
toward its own aligned target $\mathsf{mux}(\mathbf{r}_i)$, whereas $\mathcal{L}_{\mathrm{global}}$
constrains only the aggregate hidden state used to produce the answer.
We show that only the former yields a tokenwise lower bound on answer-side routing through previous latent reasoning tokens.
Our positive result is  a \emph{routing-transfer} statement.
We do not claim that answer-side routing through latent reasoning appears automatically. Instead, we isolate the regime in which the aligned multiplexed targets already have an answer-side advantage over non-reasoning context, and ask whether local distillation preserves that advantage after those targets are replaced by actual latent tokens.

Fix an answer-interface token $t$.
In \Cref{app:interp_analysis} these are the tokens
\[
\{\texttt{<EOT>},\ \texttt{The},\ \texttt{answer},\ \texttt{is},\ \texttt{:}\}.
\]
Let $\mathcal{B}_t$ denote the set of non-reasoning positions visible to $t$ that are shared by the
discrete and continuous reasoning modes, namely question tokens and answer-bridge tokens.
Recall that $\mathcal{K}\subseteq\{1,\dots,K\}$ is the set of latent-token positions whose aligned
span is non-empty.

For each $i\in\mathcal{K}$, let
\[
s_{t,i}:\Delta^{|\Vocab|-1}\to\mathbb{R}
\]
denote the attention logit assigned by token $t$ in the \emph{continuous reasoning mode} to position $i$
as a function of the represented content $f(\mathbf{x}_i)$.
For each background position $b\in\mathcal{B}_t$, let $\xi_t(b)\in\mathbb{R}$ denote its corresponding
attention logit in the same mode.
We define the total attention mass assigned by $t$ to previous latent reasoning tokens by
\begin{equation}
\label{eq:att_mass_cont_app}
A_t(\mathbf{x})
=
\frac{\sum_{i\in\mathcal{K}} \exp(s_{t,i}(f(\mathbf{x}_i)))}
{\sum_{i\in\mathcal{K}} \exp(s_{t,i}(f(\mathbf{x}_i))) + \sum_{b\in\mathcal{B}_t} \exp(\xi_t(b))}.
\end{equation}

To state the routing bound, it is enough to summarize the answer-side geometry at token $t$ by two intrinsic quantities.
First, define the aligned-target margin
\[
\Delta_t^{\mathrm{ref}}
:=
\min_{i\in\mathcal{K},\, b\in\mathcal{B}_t}
\Bigl(
 s_{t,i}(\mathsf{mux}(\mathbf{r}_i)) - \xi_t(b)
\Bigr).
\]
This is the worst-case logit margin, in the continuous reasoning mode, between an aligned multiplexed target and a background position.
Second, for $\delta>0$, define the local score-drift modulus
\[
\omega_t(\delta)
:=
\max_{i\in\mathcal{K}}
\sup_{\substack{
p\in\Delta^{|\Vocab|-1}:\\
D_{\mathrm{KL}}(\mathsf{mux}(\mathbf{r}_i)\,\|\,p)\le \delta
}}
\Bigl(
 s_{t,i}(\mathsf{mux}(\mathbf{r}_i)) - s_{t,i}(p)
\Bigr).
\]
This quantity measures the largest downward change in routing score caused by replacing the aligned target with any content inside a KL-ball of radius $\delta$.

\begin{proposition}[Local distillation preserves answer-side routing]
\label{prop:routing_preserve_final}
Fix an answer-interface token $t$ and $\delta>0$.
Define the set of well-aligned latent-token positions by
\[
\mathcal{K}_\delta
=
\left\{
i\in\mathcal{K}:
D_{\mathrm{KL}}\!\left(\mathsf{mux}(\mathbf{r}_i)\,\|\,f(\mathbf{x}_i)\right)\le \delta
\right\}.
\]
Then
\begin{equation}
\label{eq:good_count_final}
|\mathcal{K}_\delta|
\ge
|\mathcal{K}|
\left(1-\frac{\mathcal{L}_{\mathrm{local}}}{\delta}\right),
\end{equation}
and
\begin{equation}
\label{eq:routing_lb_final}
A_t(\mathbf{x})
\ge
\frac{|\mathcal{K}_\delta|\,\exp(\Delta_t^{\mathrm{ref}}-\omega_t(\delta))}
{|\mathcal{K}_\delta|\,\exp(\Delta_t^{\mathrm{ref}}-\omega_t(\delta))+|\mathcal{B}_t|}.
\end{equation}
In particular, if $\Delta_t^{\mathrm{ref}}>\omega_t(\delta)$, then every $i\in\mathcal{K}_\delta$ satisfies
\[
s_{t,i}(f(\mathbf{x}_i))>\xi_t(b),
\qquad
\forall b\in\mathcal{B}_t,
\]
so each well-aligned latent reasoning token individually outranks every background position at token $t$.
\end{proposition}

\begin{proof}
We first prove \eqref{eq:good_count_final}.

For each $i\in\mathcal{K}$, define
\[
d_i
:=
D_{\mathrm{KL}}\!\left(\mathsf{mux}(\mathbf{r}_i)\,\|\,f(\mathbf{x}_i)\right).
\]
By \eqref{eq:local_loss},
\[
\mathcal{L}_{\mathrm{local}}
=
\frac{1}{|\mathcal{K}|}\sum_{i\in\mathcal{K}} d_i.
\]
Let
\[
\mathcal{E}_\delta
:=
\{i\in\mathcal{K}: d_i>\delta\}.
\]
Each index in $\mathcal{E}_\delta$ contributes more than $\delta$ to the sum, hence
\[
\sum_{i\in\mathcal{K}} d_i
\ge
\sum_{i\in\mathcal{E}_\delta} d_i
>
|\mathcal{E}_\delta|\,\delta.
\]
Dividing by $|\mathcal{K}|$ gives
\[
\mathcal{L}_{\mathrm{local}}
>
\frac{|\mathcal{E}_\delta|}{|\mathcal{K}|}\,\delta,
\]
so
\[
|\mathcal{E}_\delta|
<
|\mathcal{K}|\,\frac{\mathcal{L}_{\mathrm{local}}}{\delta}.
\]
Since $\mathcal{K}_\delta=\mathcal{K}\setminus\mathcal{E}_\delta$, we obtain
\[
|\mathcal{K}_\delta|
=
|\mathcal{K}|-|\mathcal{E}_\delta|
\ge
|\mathcal{K}|
\left(1-\frac{\mathcal{L}_{\mathrm{local}}}{\delta}\right),
\]
which proves \eqref{eq:good_count_final}.

We now prove \eqref{eq:routing_lb_final}.
Fix any $i\in\mathcal{K}_\delta$.
By definition of $\mathcal{K}_\delta$,
\[
D_{\mathrm{KL}}\!\left(\mathsf{mux}(\mathbf{r}_i)\,\|\,f(\mathbf{x}_i)\right)\le \delta.
\]
Therefore, by definition of $\omega_t(\delta)$,
\[
s_{t,i}(f(\mathbf{x}_i))
\ge
s_{t,i}(\mathsf{mux}(\mathbf{r}_i))
-
\omega_t(\delta).
\]
By definition of $\Delta_t^{\mathrm{ref}}$, for every $b\in\mathcal{B}_t$,
\[
s_{t,i}(\mathsf{mux}(\mathbf{r}_i))
\ge
\xi_t(b)+\Delta_t^{\mathrm{ref}}.
\]
Combining the two displays gives
\begin{equation}
\label{eq:latent_bg_compare_final}
s_{t,i}(f(\mathbf{x}_i))
\ge
\xi_t(b)+\Delta_t^{\mathrm{ref}}-\omega_t(\delta),
\qquad
\forall b\in\mathcal{B}_t.
\end{equation}

Choose $b_t^\star\in\mathcal{B}_t$ satisfying
\[
\xi_t(b_t^\star)=\max_{b\in\mathcal{B}_t}\xi_t(b).
\]
Applying \eqref{eq:latent_bg_compare_final} with $b=b_t^\star$ gives
\[
s_{t,i}(f(\mathbf{x}_i))
\ge
\xi_t(b_t^\star)+\Delta_t^{\mathrm{ref}}-\omega_t(\delta).
\]
Exponentiating both sides,
\[
\exp(s_{t,i}(f(\mathbf{x}_i)))
\ge
\exp(\Delta_t^{\mathrm{ref}}-\omega_t(\delta))\,
\exp(\xi_t(b_t^\star)).
\]
This holds for every $i\in\mathcal{K}_\delta$.
Summing over $i\in\mathcal{K}_\delta$,
\[
\sum_{i\in\mathcal{K}} \exp(s_{t,i}(f(\mathbf{x}_i)))
\ge
|\mathcal{K}_\delta|\,
\exp(\Delta_t^{\mathrm{ref}}-\omega_t(\delta))\,
\exp(\xi_t(b_t^\star)).
\]
On the other hand, by maximality of $\xi_t(b_t^\star)$,
\[
\sum_{b\in\mathcal{B}_t}\exp(\xi_t(b))
\le
|\mathcal{B}_t|\,\exp(\xi_t(b_t^\star)).
\]
Substituting these two bounds into \eqref{eq:att_mass_cont_app} gives \eqref{eq:routing_lb_final}.
The final claim follows directly from \eqref{eq:latent_bg_compare_final}.
\end{proof}

\Cref{prop:routing_preserve_final} shows that the local objective controls how many positions stay inside a KL-ball around their aligned targets, while the sign and magnitude of $\Delta_t^{\mathrm{ref}}-\omega_t(\delta)$ determine whether the answer-side preference survives inside that ball.
The count bound becomes informative once $\delta>\mathcal{L}_{\mathrm{local}}$, but the proposition itself does not assume any positivity condition on the margin.

To connect this statement back to the \emph{discrete reasoning mode}, let $\bar\ell_t(v)$ denote the attention logit from token $t$ to a discrete reasoning position $v$.
For each aligned span $\mathbf{r}_i=(\mathbf{r}_i^1,\dots,\mathbf{r}_i^{|\mathbf{r}_i|})$, define the corresponding span-level routing score by
\begin{equation}
\label{eq:span_score_disc_app}
\bar s_{t,i}
:=
\log \sum_{u=1}^{|\mathbf{r}_i|} \exp(\bar\ell_t(\mathbf{r}_i^u)).
\end{equation}
Thus, $\bar s_{t,i}$ is the log-sum-exp score assigned by token $t$ to the entire aligned span $\mathbf{r}_i$ in the discrete reasoning mode.
For each background position $b\in\mathcal{B}_t$, let $\bar\xi_t(b)\in\mathbb{R}$ denote its attention logit in the discrete reasoning mode.
Now define the discrete routing margin
\[
\Delta_t^{\mathrm{disc}}
:=
\min_{i\in\mathcal{K},\, b\in\mathcal{B}_t}
\bigl(\bar s_{t,i}-\bar\xi_t(b)\bigr).
\]
This is the formal version of the answer-side preference for aligned reasoning spans studied in prior routing analyses~\citep{zhang2025reasoning,tutek2025measuring}.
Also define the calibration gap
\[
\Gamma_t
:=
\max\left\{
\max_{i\in\mathcal{K}}\left|s_{t,i}(\mathsf{mux}(\mathbf{r}_i))-\bar s_{t,i}\right|,
\max_{b\in\mathcal{B}_t}\left|\xi_t(b)-\bar\xi_t(b)\right|
\right\}.
\]
Then, for every $i\in\mathcal{K}$ and $b\in\mathcal{B}_t$,
\[
s_{t,i}(\mathsf{mux}(\mathbf{r}_i)) - \xi_t(b)
\ge
\bar s_{t,i} - \bar\xi_t(b) - 2\Gamma_t,
\]
hence
\[
\Delta_t^{\mathrm{ref}}
\ge
\Delta_t^{\mathrm{disc}} - 2\Gamma_t.
\]
Substituting this into \Cref{prop:routing_preserve_final} yields
\[
A_t(\mathbf{x})
\ge
\frac{|\mathcal{K}_\delta|\,\exp(\Delta_t^{\mathrm{disc}}-2\Gamma_t-\omega_t(\delta))}
{|\mathcal{K}_\delta|\,\exp(\Delta_t^{\mathrm{disc}}-2\Gamma_t-\omega_t(\delta))+|\mathcal{B}_t|}.
\]
This is the sense in which local distillation preserves answer-side routing: a routing preference present in the discrete trace transfers to the continuous mode provided the aligned targets retain a positive calibrated margin and the actual distilled tokens do not drift far enough to erase it.

We now contrast this with trajectory-level supervision alone. Let $\mathbf{h}_t^\star\in\mathbb{R}^d$ denote the answer-interface hidden state in the discrete reasoning mode at token $t$.
Write the hidden state in the continuous reasoning mode as
\begin{equation}
\label{eq:h_cont_mode_final}
\mathbf{h}_t(\mathbf{x})
=
\sum_{i\in\mathcal{K}} a_{t,i}(\mathbf{x})\,\mathbf{v}_{t,i}
+
\sum_{b\in\mathcal{B}_t} a_{t,b}(\mathbf{x})\,\mathbf{v}_{t,b},
\end{equation}
where $a_{t,i}(\mathbf{x})$ and $a_{t,b}(\mathbf{x})$ are the attention weights at token $t$, and
$\mathbf{v}_{t,i},\mathbf{v}_{t,b}\in\mathbb{R}^d$ are the corresponding value vectors.
Within this abstraction,
\[
A_t(\mathbf{x}) = \sum_{i\in\mathcal{K}} a_{t,i}(\mathbf{x}).
\]

\begin{proposition}[Global distillation alone does not identify routing]
\label{prop:global_no_routing_final}
Fix an answer-interface token $t$.
Assume that there exist coefficients $(\lambda_b)_{b\in\mathcal{B}_t}$ with
\[
\lambda_b\ge 0,
\qquad
\sum_{b\in\mathcal{B}_t}\lambda_b=1,
\]
such that
\begin{equation}
\label{eq:bg_realizability_final}
\left\|
\sum_{b\in\mathcal{B}_t}\lambda_b \mathbf{v}_{t,b}
-
\mathbf{h}_t^\star
\right\|_2
\le
\varepsilon_0.
\end{equation}
Then for every $\eta\in(0,1)$, there exists a choice of attention weights at token $t$ such that
\[
A_t(\mathbf{x})=\eta
\]
and
\[
\|\mathbf{h}_t(\mathbf{x})-\mathbf{h}_t^\star\|_2
\le
\varepsilon_0 + C_t\eta,
\]
where
\[
C_t
=
\left\|
\sum_{b\in\mathcal{B}_t}\lambda_b \mathbf{v}_{t,b}
\right\|_2
+
\max_{i\in\mathcal{K}}\|\mathbf{v}_{t,i}\|_2.
\]
Consequently, $\mathcal{L}_{\mathrm{global}}$ alone does not imply any strictly positive lower bound on answer-side attention to previous latent reasoning tokens.
\end{proposition}

\begin{proof}
Define
\[
\bar{\mathbf{h}}_t
:=
\sum_{b\in\mathcal{B}_t}\lambda_b \mathbf{v}_{t,b}.
\]
By \eqref{eq:bg_realizability_final},
\[
\|\bar{\mathbf{h}}_t-\mathbf{h}_t^\star\|_2\le \varepsilon_0.
\]

Fix any $\eta\in(0,1)$.
Choose the attention weights at token $t$ by
\[
a_{t,i}^{(\eta)}(\mathbf{x}) := \frac{\eta}{|\mathcal{K}|},
\qquad i\in\mathcal{K},
\]
and
\[
a_{t,b}^{(\eta)}(\mathbf{x}) := (1-\eta)\lambda_b,
\qquad b\in\mathcal{B}_t.
\]
These weights are nonnegative and satisfy
\[
\sum_{i\in\mathcal{K}} a_{t,i}^{(\eta)}(\mathbf{x})
+
\sum_{b\in\mathcal{B}_t} a_{t,b}^{(\eta)}(\mathbf{x})
=
\eta + (1-\eta)\sum_{b\in\mathcal{B}_t}\lambda_b
=
1.
\]
Hence they define a valid attention distribution in the attention-mixture abstraction.

By construction, the total attention mass on previous latent reasoning tokens is exactly
\[
A_t(\mathbf{x})
=
\sum_{i\in\mathcal{K}} a_{t,i}^{(\eta)}(\mathbf{x})
=
\eta.
\]

Substituting the chosen weights into \eqref{eq:h_cont_mode_final} gives
\[
\mathbf{h}_t^{(\eta)}(\mathbf{x})
=
\sum_{i\in\mathcal{K}} \frac{\eta}{|\mathcal{K}|}\,\mathbf{v}_{t,i}
+
\sum_{b\in\mathcal{B}_t}(1-\eta)\lambda_b\,\mathbf{v}_{t,b}.
\]
Using the definition of $\bar{\mathbf{h}}_t$, we may rewrite this as
\[
\mathbf{h}_t^{(\eta)}(\mathbf{x})
=
(1-\eta)\bar{\mathbf{h}}_t
+
\frac{\eta}{|\mathcal{K}|}\sum_{i\in\mathcal{K}}\mathbf{v}_{t,i}.
\]
Subtracting $\mathbf{h}_t^\star$ yields
\[
\mathbf{h}_t^{(\eta)}(\mathbf{x})-\mathbf{h}_t^\star
=
(\bar{\mathbf{h}}_t-\mathbf{h}_t^\star)
-
\eta\,\bar{\mathbf{h}}_t
+
\frac{\eta}{|\mathcal{K}|}\sum_{i\in\mathcal{K}}\mathbf{v}_{t,i}.
\]
Taking norms and applying the triangle inequality,
\[
\|\mathbf{h}_t^{(\eta)}(\mathbf{x})-\mathbf{h}_t^\star\|_2
\le
\|\bar{\mathbf{h}}_t-\mathbf{h}_t^\star\|_2
+
\eta\|\bar{\mathbf{h}}_t\|_2
+
\left\|\frac{\eta}{|\mathcal{K}|}\sum_{i\in\mathcal{K}}\mathbf{v}_{t,i}\right\|_2.
\]
For the last term,
\[
\left\|\frac{\eta}{|\mathcal{K}|}\sum_{i\in\mathcal{K}}\mathbf{v}_{t,i}\right\|_2
\le
\frac{\eta}{|\mathcal{K}|}\sum_{i\in\mathcal{K}}\|\mathbf{v}_{t,i}\|_2
\le
\eta \max_{i\in\mathcal{K}}\|\mathbf{v}_{t,i}\|_2.
\]
Combining the preceding two displays with
$\|\bar{\mathbf{h}}_t-\mathbf{h}_t^\star\|_2\le\varepsilon_0$, we obtain
\[
\|\mathbf{h}_t^{(\eta)}(\mathbf{x})-\mathbf{h}_t^\star\|_2
\le
\varepsilon_0
+
\eta
\left(
\|\bar{\mathbf{h}}_t\|_2
+
\max_{i\in\mathcal{K}}\|\mathbf{v}_{t,i}\|_2
\right).
\]
By the definition of $C_t$, this is
\[
\|\mathbf{h}_t^{(\eta)}(\mathbf{x})-\mathbf{h}_t^\star\|_2
\le
\varepsilon_0 + C_t\eta,
\]
which proves the claim.
\end{proof}

\Cref{prop:global_no_routing_final} says that matching the answer-interface hidden state does not identify the routing pattern. The same observation applies to $\mathcal{L}_{\mathrm{CE}}$: if two routing patterns induce the same answer-interface hidden state, then they incur the same answer loss.

\Cref{prop:routing_preserve_final,prop:global_no_routing_final} separate what the local and global objectives control. Local distillation yields a tokenwise lower bound on answer-side attention exactly when the aligned targets retain a positive answer-side margin relative to their local score drift. By contrast, global supervision alone does not provide any analogous guarantee. The endpoint hidden state can be matched while the attention mass on previous latent reasoning tokens is made arbitrarily small.

\section{Benchmark details}
\label{app:benchmarks}
\label{app:search_details}

\subsection{MNNS task}

The Minimum Non-Negative Sum (MNNS) task~\citep{gozeten2026continuous} takes as input $H$ positive integers $a_1,\ldots,a_H$ and asks for the minimum value of $\sum_{k=1}^H \sigma_k a_k \geq 0$ over all sign assignments $\sigma_k\in\{-1,+1\}$.
This is equivalent to finding the partition of $\{a_1,\ldots,a_H\}$ into two subsets with minimal non-negative difference, a variant of the subset-sum problem \citep{karp2009reducibility}.

\paragraph{Graph construction.}
We define a layered directed graph $G=(\mathcal{N},E)$ with
\[
\mathcal{N} = \{(k,z) : k\in\{0,\ldots,H\},\; z\text{ is a partial sum reachable at depth }k\},
\]
and edges $((k,z),(k{+}1,z+a_{k+1})),\;((k,z),(k{+}1,z-a_{k+1}))\in E$.
The source is $s=(0,0)$.
At depth $k$, the BFS frontier $F_k$ contains all partial-sum states discovered for the first time, and the discovered set $U_k$ accumulates all states seen up to depth $k$.
The answer is the minimum non-negative $z$ such that $(H,z)\in U_H$.

\paragraph{Data and vocabulary.}
For $H{=}4$ digits drawn from $\{1,\ldots,9\}$, the vocabulary consists of integers in $[-S,S]$ (with $S$ chosen so all reachable partial sums are covered) plus special tokens.
The input is formatted as $\langle\textsc{bos}\rangle\; a_1\; a_2\;\ldots\; a_H\;\to$, and the output is the optimal sum value followed by $\langle\textsc{eos}\rangle$.
Permutations of the same integer multiset are assigned to the same data split (80\%/20\% train/val) to prevent data leakage.

\paragraph{Architecture.}
We use a 2-layer, 2-head GPT-2 model with embedding dimension $d{=}32$, trained from scratch with AdamW \citep{loshchilov2018decoupled} (learning rate $10^{-4}$, no weight decay).
Each of the $H{-}1$ intermediate steps corresponds to one latent token; the final discrete token produces the answer.

\subsection{Game of 24 task}
The Game of 24~\citep{yao2023tree} is a classic arithmetic puzzle: given a set of numbers, the goal is to combine them using arithmetic operations to reach the target value 24. We formulate a sequential variant that naturally maps to a layered reachability problem.

\paragraph{Task formulation.}
Given $C$ cards with values drawn from $\{1,\ldots,D\}$ and an operator set $\mathcal{O}$, the model must determine whether the target value~24 is reachable by processing the cards strictly left to right. Starting with the first card as accumulator, at each step $k\in\{1,\ldots,C{-}1\}$ one applies an operation $\circ_k\in\mathcal{O}$ to produce $A_{k}=A_{k-1}\circ_k d_{k+1}$, where $d_{k+1}$ is the $(k{+}1)$-th card. The answer is $y=\mathbf{1}[24\in A_{C-1}]$, where $A_{C-1}$ is the set of accumulated values reachable after all $C$ cards over all operation sequences.

\paragraph{Graph construction.}
This defines a layered directed graph $G=(\mathcal{N},E)$. The node set at depth~$k$ consists of all intermediate values reachable after incorporating $k{+}1$ cards:
\[
\mathcal{N}_k = \{v : v\text{ is an accumulated value reachable at step }k\}.
\]
Edges connect each node $v\in\mathcal{N}_k$ to nodes $\{v\circ d_{k+2}:\circ\in\mathcal{O}\}\cap\mathcal{N}_{k+1}$, and the source is $s=d_1$. At each step, the BFS frontier $F_k$ records the set of newly discovered accumulated values, so the uniform multiplexed target $\mathsf{mux}(\mathbf{r}_k)$ is the distribution over the current frontier.

\paragraph{Configuration.}
We use $C{=}5$ cards, digit range $\{1,\ldots,5\}$, and operator set $\mathcal{O}=\{+,-,\times\}$. The dataset is balanced (50\% reachable, 50\% unreachable). We assign all permutations of the same card multiset to the same split. Because order affects the left-to-right process, this split prevents memorization of card multisets while still evaluating order-sensitive reasoning. Training uses 3 random seeds.

\paragraph{Architecture.}
We use the same 2-layer, 2-head GPT-2 model with $d{=}32$ as for MNNS. Each of the $C{-}1=4$ fold steps corresponds to one latent token; the final discrete token produces the YES/NO answer.

\FloatBarrier
\section{Method and training details}
\label{app:method}

\subsection{Implementation details}
\label{app:implementation}

\Cref{tab:lora_config,tab:training_hyperparams} summarize the hyperparameters used for all \ours{} experiments. We follow the same experimental protocol as CODI~\citep{shen2025codi}. All models were trained using a single H100 GPU with 96 GB of VRAM. Experiments with GPT-2 and LLaMA 1B took around 24 hours, while the experiments with larger backbones (LLaMA 3B/8B) ran for 2--3 days to complete.

\begin{table}[t!]
\centering
\caption{LoRA adapter configuration (shared across all models).}
\small
\begin{tabular}{lc}
\toprule
\textbf{Hyperparameter} & \textbf{Value} \\
\midrule
LoRA rank $r$ & 128 \\
LoRA alpha & 32 \\
LoRA dropout & 0.1 \\
\bottomrule
\end{tabular}
\label{tab:lora_config}
\end{table}

\begin{table}[t!]
\centering
\caption{Training hyperparameters for \ours{}. Method-specific parameters are listed in the top block; standard optimization settings are in the bottom block.}
\small
\setlength{\tabcolsep}{3.2pt}
\begin{tabular}{l cccccc}
\toprule
& \multicolumn{2}{c}{\textbf{GPT-2}} & \multicolumn{2}{c}{\textbf{LLaMA 3.2 1B}} & \textbf{LLaMA 3.2 3B} & \textbf{LLaMA 3.1 8B} \\
\cmidrule(lr){2-3} \cmidrule(lr){4-5} \cmidrule(lr){6-6} \cmidrule(lr){7-7}
\textbf{Hyperparameter} & Aug & NL & Aug & NL & Aug & Aug \\
\midrule
\multicolumn{7}{l}{\textit{Method-specific}} \\
\midrule
Continuous tokens $K$ & 6 & 6 & 6 & 6 & 6 & 6 \\
Weighting function & sin. & sin. & geo. & sin. & geo. & geo. \\
Decay rate $\rho$ (geo.) & --- & --- & 0.9 & --- & 0.9 & 0.9 \\
Positional scale $\lambda$ (sin.) & 1.0 & 1.0 & --- & 1.0 & --- & --- \\
Temperature $\tau$ & 1.0 & 1.0 & 1.0 & 1.0 & 1.0 & 1.0 \\
Chunking strategy & rand. & rand. & rand. & rand. & rand. & rand. \\
Local loss weight $\beta$ & 1.0 & 1.0 & 1.0 & 1.0 & 1.0 & 1.0 \\
Global distill.\ weight $\gamma$ & 1.0 & 1.0 & 20.0 & 20.0 & 20.0 & 20.0 \\
Answer loss weight & 1.0 & 1.0 & 1.0 & 1.0 & 1.0 & 1.0 \\
Ref.\ answer loss weight & 1.0 & 1.0 & 1.0 & 1.0 & 1.0 & 1.0 \\
Projection dim & 768 & 768 & 2048 & 2048 & 3072 & 4096 \\
Layer-wise std norm & \checkmark & \checkmark & \checkmark & \checkmark & \checkmark & \checkmark \\
\midrule
\multicolumn{7}{l}{\textit{Optimization}} \\
\midrule
Optimizer & \multicolumn{6}{c}{AdamW} \\
LR scheduler & \multicolumn{6}{c}{cosine} \\
Warmup ratio & \multicolumn{6}{c}{0.03} \\
Effective batch size & \multicolumn{6}{c}{128} \\
Learning rate & 3e-3 & 3e-3 & 8e-4 & 8e-4 & 3e-4 & 1e-4 \\
Weight decay & 0.01 & 0.01 & 0.1 & 0.1 & 0.1 & 0.1 \\
Gradient clipping & 1.0 & 1.0 & 2.0 & 2.0 & 2.0 & 2.0 \\
Epochs & 40 & 40 & 10 & 10 & 8 & 6 \\
\bottomrule
\end{tabular}
\label{tab:training_hyperparams}
\end{table}

\paragraph{\ours{}$^{*}$ ($K{=}24$) configuration.}
The parallel-decoding variant \ours{}$^{*}$ reported in \Cref{tab:main_results} uses $K{=}24$ latent tokens generated via $T{=}3$ Jacobi iterations. We use uniform token chunking. Geometric positional weighting is applied with decay $0.9$. The local loss weight is $\beta{=}1.0$ and the global distillation weight is $\gamma{=}20.0$. Optimization settings match the LLaMA~3.2 1B column of \Cref{tab:training_hyperparams}.

\subsection{Details of probing for positional weighting}
\label{app:probe_positional_weighting}

In \Cref{sec:ablations}, to study the role of positional weighting, we trained an MLP probe on discrete reasoning spans \(\mathbf r_i=(r_i^1,\dots,r_i^{S_i})\). For each non-empty span \(\mathbf r_i\), we construct the same multiplexed target as in \eqref{eq:our_mux}.

The probe takes \(\mathsf{mux}(\mathbf r_i)\) as input and predicts the original span \(\mathbf r_i\).
We use a 5-layer MLP with hidden sizes \(1024,512,256,512,1024\) and GELU
\citep{hendrycks2016gaussian} activations, without input normalization. We set the maximum sequence
length to 128, strip the delimiters \texttt{<<} and \texttt{>>} from each extracted step, and train
for 20 epochs with batch size 128 and learning rate \(10^{-3}\). The data are extracted from
GSM8K-AUG and split into 901,661 training spans and 100,185 evaluation spans using a 0.1 test split.
For geometric weighting, we use \(\rho=0.9\). For sinusoidal weighting, we use \(\tau=1\). For rotary weighting, we use \(\mathrm{base}=1000\).

\subsection{Details of span-level alignments}
\label{app:alignment}

The main text only assumes an order-preserving alignment between the
\(M\) discrete reasoning spans and the \(K\) latent-token slots.
Let \((\tilde{\mathbf r}_1,\ldots,\tilde{\mathbf r}_M)\) denote the aligned spans used for local supervision, where each non-empty \(\tilde{\mathbf r}_i\) is a contiguous block of the original reasoning spans and the original order is preserved. When \(M \le K\), all alignment variants reduce to the same prefix assignment: \(\tilde{\mathbf r}_i=\mathbf r_i\) for \(i \le M\), and the remaining \(K-M\) slots are left empty. The differences arise only in the overfull regime \(M > K\), which we summarize below.

\paragraph{No chunking.}
Assign one span to each latent token until one of the two sequences ends. Equivalently, \(\tilde{\mathbf r}_i=\mathbf r_i\) for \(i \le \min(M,K)\). If \(M > K\), the remaining \(M-K\) spans are discarded, so this is lossless only when \(M \le K\).

\paragraph{Deterministic chunking.}
When \(M > K\), partition the \(M\) reasoning spans into \(K\) contiguous groups with roughly equal sizes. Writing \(M=qK+r\) with \(0 \le r < K\), the first \(K-r\) groups have size \(q\) and the last \(r\) groups have size \(q+1\). Equivalently, group sizes differ by at most one, with extra spans assigned to later groups.

\paragraph{Random chunking.}
When \(M > K\), sample \(K-1\) cut points uniformly without replacement from \(\{1,\dots,M-1\}\), sort them, and use the induced intervals to form \(K\) positive contiguous groups. This yields a random monotone partition of the \(M\) spans into \(K\) chunks. This is the default variant used in our main experiments. Randomness is resampled during training, so the model sees multiple valid local segmentations of the same reasoning trace without changing the answer target.
 
\section{Limitations and broader impact}
\label{app:limitations}

\paragraph{Limitations.} Our losslessness guarantees are stated under exact arithmetic. In finite precision, very long spans or large vocabularies may approach the separation boundary analyzed in \Cref{app:finite_precision}, though the analysis confirms that losslessness is preserved in practical regimes (standard span lengths and float32 precision). Our empirical evaluation focuses on mathematical reasoning and parallel search tasks. This is the established evaluation setting adopted by prior latent reasoning methods~\citep{shen2025codi, wei2026simcot, kuzina2026kava, gozeten2026continuous}. Extending \ours{} to broader reasoning domains such as multi-hop question answering, code generation, and open-ended planning is a natural next step.

\paragraph{Broader impact.} Compressing reasoning into fewer latent tokens can reduce inference cost. A concern with latent reasoning is reduced interpretability. Users cannot easily audit intermediate steps~\citep{kuzina2026kava}. \ours{} partially addresses this, since each latent token can be decoded into human-readable content through the LM head. Still, decoded tokens are approximate, not verbatim reasoning, so users should not treat them as ground truth. More broadly, more efficient reasoning inherits the risks of the underlying models, which can produce incorrect, biased, or overconfident outputs.

\end{document}